\documentclass[11pt]{article}

\usepackage{amsthm}
\usepackage{jmlr2e}
\usepackage{mlmath}
\usepackage[shortlabels]{enumitem}
\usepackage{tabularx}
\usepackage{multirow}
\usepackage{booktabs}
\usepackage[caption=false]{subfig}

\usepackage{xcolor}

\newtheorem{innercustomgeneric}{\customgenericname}
\providecommand{\customgenericname}{}

\newtheorem{thm}{Theorem}
\newtheorem{cor}{Corollary}
\newtheorem{defi}{Definition}
\newtheorem{lem}{Lemma}

\newtheorem{rmk}{Remark}
\newtheorem{assump}{Assumption}

\begin{document}

\title{On the exact computation of linear frequency principle dynamics and its generalization\thanks{Authors are listed alphabetically.}}

\author{%
    \name Tao\ Luo \email luotao41@sjtu.edu.cn \\
    \name Zheng Ma \email zhengma@sjtu.edu.cn \\
    \name Zhi-Qin John Xu \email xuzhiqin@sjtu.edu.cn \\
    \name Yaoyu Zhang \email zhyy.sjtu@sjtu.edu.cn \\
    \addr School of Mathematical Sciences, Institute of Natural Sciences, MOE-LSC and Qing Yuan Research Institute,
    Shanghai Jiao Tong University, Shanghai, 200240, P.R. China}

\editor{}
\date{\today}


\maketitle

\begin{abstract}
    Recent works show an intriguing phenomenon of Frequency Principle (F-Principle) that deep neural networks (DNNs) fit the target function from low to high frequency during the training, which provides insight into the training and generalization behavior of DNNs in complex tasks. In this paper, through analysis of an infinite-width two-layer NN in the neural tangent kernel (NTK) regime, we derive the exact differential equation, namely Linear Frequency-Principle (LFP) model, governing the evolution of NN output function in the frequency domain during the training. Our exact computation applies for general activation functions with no assumption on size and distribution of training data. This LFP model unravels that higher frequencies evolve polynomially or exponentially slower than lower frequencies depending on the smoothness/regularity of the activation function. We further bridge the gap between training dynamics and generalization by proving that LFP model implicitly minimizes a Frequency-Principle norm (FP-norm) of the learned function, by which higher frequencies are more severely penalized depending on the inverse of their evolution rate. Finally, we derive an \textit{a priori} generalization error bound controlled by the FP-norm of the target function, which provides a theoretical justification for the empirical results that DNNs often generalize well for low frequency functions.
\end{abstract}
\section{Introduction}

Recently, an intriguing phenomenon of Frequency Principle (F-Principle) sheds light on understanding the success and failure of DNNs. It is discovered that, in various settings, deep neural networks (DNNs) fit the target function from low to high frequency during the training \citep{xu_training_2018,rahaman2018spectral,xu2019frequency}. The F-Principle implies that DNNs are biased toward a low-frequency fitting of the training data, which provides hints to the generalization of DNNs in practice \citep{xu2019frequency,ma2020slow}. The F-Principle provided valuable guidance in designing DNN-based algorithms \citep{cai_phase_2019,biland2019frequency,jagtap_adaptive_2020,liu2020multi}. The convergence behavior from low to high frequency is also consistent with other empirical studies showing that DNNs increase the complexity of the output function during the training process quantified by various complexity measures \citep{arpit2017closer,valle2018deep,mingard2019neural,nakkiran2019sgd}. 

Despite of the rich practical implications of the F-Principle, the gap between F-Principle training dynamics and success or failure of DNNs (i.e., generalization performance) remains a key theoretical challenge. Bridging this gap requires an exact characterization of the F-Principle accounting for the conditions of overparameterization and finite training data in practice, which is not provided by existing theories \citep{basri2019convergence,bordelon2020spectrum,cao_towards_2020,e2019machine} 

In this work, based on mean-field analysis of an infinite-width two-layer NN in the NTK regime, we derive the exact differential equation, namely Linear Frequency-Principle (LFP) model, governing the evolution of NN output function in the frequency domain during the training. Our exact computation applies for general activation functions with no assumption on size and distribution of training data. Our LFP model rigorously characterizes the F-Principle and unravels that higher frequencies evolves polynomially or exponentially slower than lower frequencies depending on the smoothness/regularity of the activation function. We further prove that LFP dynamics implicitly minimizes a Frequency-Principle norm (FP-norm), by which higher frequencies are more severely penalized depending on the inverse of their evolution rate. Specifically, for 1-d regression problems, this optimization yields linear spline, cubic spline or their combination depending on parameter initialization for ReLU activation. Finally, we derive an \textit{a priori} generalization error bound controlled by the FP-norm of the target function, which provides a unified qualitative explanation to the success and failure of DNNs. These three results are demonstrated by Theorems \ref{mainthm}, \ref{thm..EquivalenceDynamicsMinimization} and \ref{thm:priorierror}, respectively. For a better understanding of how we arrive three theorems, we depict the sketch of proofs for each theorem in Fig. \ref{sketch}.

The structure of the paper is organized as follows. We review related works in Section~\ref{sec:relatedworks}. Before we present our results, we introduce some preliminaries in Section~\ref{sec:Preliminaries}. Then, we show the exact computation of the LFP model in Section~\ref{sec:lfpmodel}. In Section \ref{sec:Explicitizing-the-implicit}, we explicitize the implicit bias of the F-Principle by proving the equivalence between the LFP model and an optimization problem. Further, we estimate an \textit{a priori} generalization error bound for the LFP model in Section~\ref{FPapriori}. In Section~\ref{sec:exps}, we use experiments to validate the effectiveness of the LFP model for ReLU and Tanh activation functions. Finally, we present conclusions and discussion in Section~\ref{sec:discussion}.
\begin{center}
\begin{figure}
\begin{centering}
\includegraphics[scale=0.75]{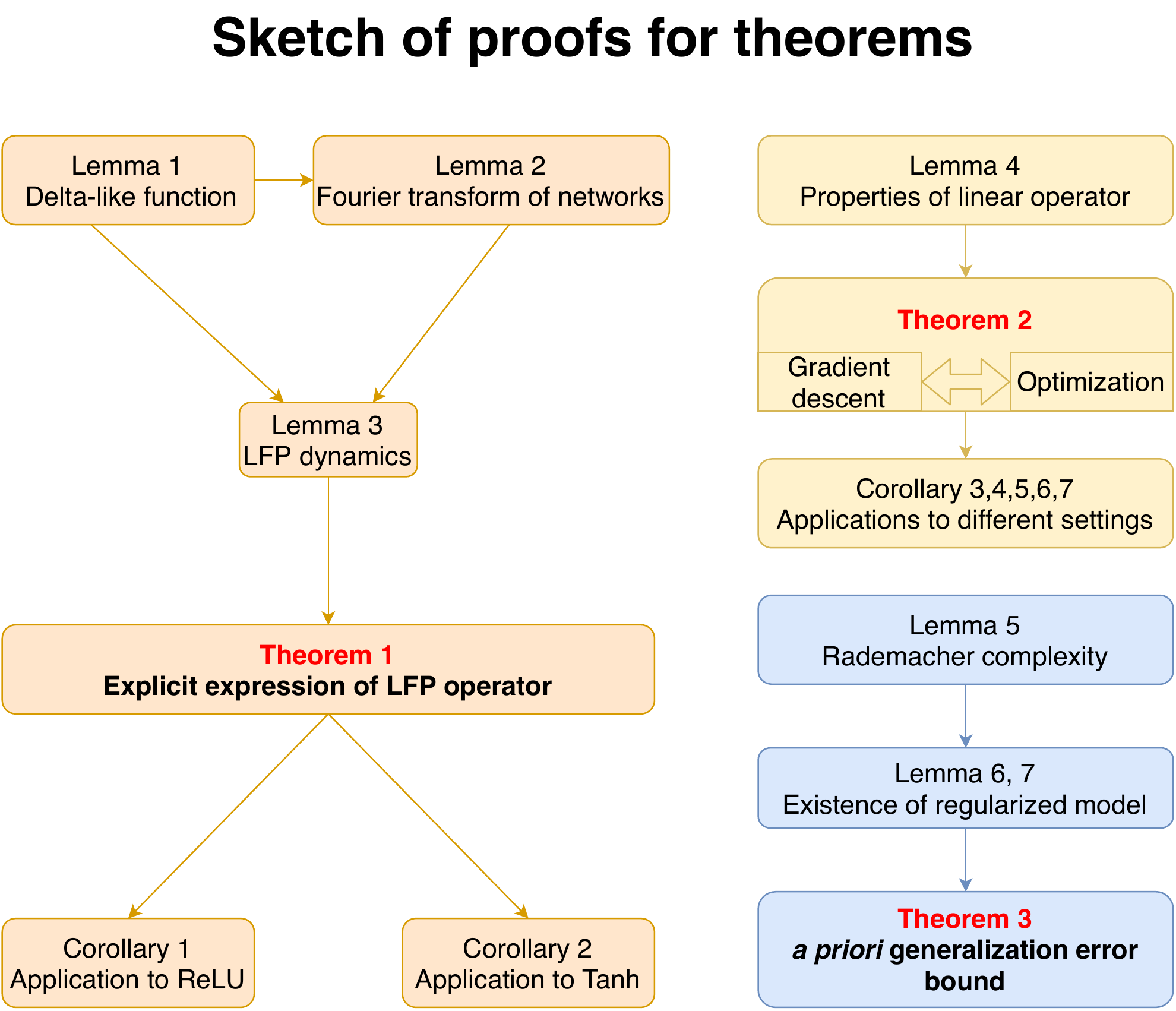}
\par\end{centering}
\caption{Main theoretical results and sketch of proofs. \label{sketch} }
\end{figure}
\par\end{center}

\section{Related works}\label{sec:relatedworks}
A series of works have devoted to reveal underlying mechanisms of the F-Principle. \citet{xu2018understanding} and \citet{xu2019frequency} show that the gradient of low-frequency loss exponentially dominates that of high-frequency ones when parameters are small for DNNs with tanh activation. A key mechanism of F-Principle has been pointed out that the low-frequency dominant gradient is a consequence of the smoothness of the activation function. \citet{rahaman2018spectral} later extend the framework of tanh activation function to the ReLU activation function. \citet{luo2019theory} estimate the dynamics of different frequency components of the loss function for arbitrary data distribution with mild regularity assumption and sufficient large size of training data size.

At the same time with our work, several parallel works also analyze the F-Principle (or spectral bias) in the NTK regime.  \citet{basri2019convergence} and \citet{cao_towards_2020} estimate the convergence speed of each frequency for two-layer wide ReLU networks in the NTK regime with the assumption of a sufficient large size of training data uniformly distributed on a hyper-sphere. \citet{basri2020frequency} release the assumption on data distribution to a nonuniform one, which is restricted to two-dimensional sphere, and they derive a similar frequency bias for two-layer wide ReLU networks in the NTK regime. \citet{bordelon2020spectrum} study the dependence of the spectral bias on the sample size. Several other works also focus on studying the spectral of Gram matrix in the NTK regime \citep{arora2019fine,yang2019fine}. 

In this work, our exact derivation of linear frequency principle dynamics makes no assumption about the distribution and size of training data. It is the first NN-derived quantitative model that not only shows the origin of the F-Principle but also can be used to analyze both its training and generalization consequence \footnote{A previous incomplete version of this work is released at arXiv \citep{zhang2019explicitizing}.}.

\section{Preliminaries} \label{sec:Preliminaries}
We provide some preliminary results in this section. 
\subsection{Fourier transforms}

The Fourier transform of a function $g$ is denoted by $\hat{g}$ or $\fF[g]$. The one-dimensional Fourier transform and its inverse transform is defined by
\begin{align}
    \fF[g](\xi)    & = \fF_{x\to\xi}[g](\xi) =\int_{\sR}g(x)\E^{-2\pi\I\xi x}\diff{x},    \\
    \fF^{-1}[g](x) & = \fF^{-1}_{\xi\to x}[g](x) =\int_{\sR}g(\xi)\E^{2\pi\I\xi x}\diff{\xi}.
\end{align}
Based on these, we define the high-dimensional Fourier transform and its inverse transform:
\begin{align}
    \fF[g](\vxi)     & = \fF_{\vx\to\vxi}[g](\vxi) = \int_{\sR^d}g(\vx)\E^{-2\pi\I\vxi\cdot\vx}\diff{\vx},  \\
    \fF^{-1}[g](\vx) & = \fF^{-1}_{\vxi\to\vx}[g](\vx) = \int_{\sR^d}g(\vxi)\E^{2\pi\I\vxi\cdot\vx}\diff{\vxi}.
\end{align}
Here and latter, the vector $\vx\in\sR^d$ and
$\vx^{\perp}=\vx-(\vx\cdot\hat{\vw})\hat{\vw}$ for a given
$\vw\in\sR^d\backslash \{\vzero \}$ with $\hat{\vw}=\vw/\norm{\vw}$. 
We list some useful and well-known results for one-dimensional as well as high-dimensional Fourier transforms in Appendix \ref{FTtable}.
To compute rigorously, we work in the theory of tempered distributions. Let $\fS(\sR^d)$ be the Schwartz space on $\sR^d$ and $\fS'(\sR^{d}):=(\fS(\sR^d))'$ is the space of tempered distributions. 
For any Schwartz function $\phi\in\fS(\sR^d)$ and any tempered distribution $\psi\in \fS'(\sR^{d})$, we write the pairing $\langle \psi,\phi\rangle:=\langle \psi,\phi\rangle_{\fS'(\sR^d),\fS(\sR^d)}=\psi(\phi)$, and then the Fourier transform of $\psi$ is defined by
\begin{equation}
    \langle \fF[\psi],\phi\rangle
    = \langle \psi,\fF[\phi]\rangle.
\end{equation}

\subsection{High-dimensional delta-like function}

\begin{defi}
    Given a nonzero vector $\vw\in\sR^d$, we define the delta-like function
    $\delta_{\vw}: \fS(\sR^d)\to\sR$ such that for any $\phi\in \fS(\sR^d)$,
    \begin{equation}
        \langle\delta_{\vw},\phi\rangle=\int_{\sR}
        \phi(y\vw)\diff{y}.
    \end{equation}
\end{defi}
\begin{lem}[Scaling property of delta-like function] Given any nonzero vector $\vw\in\sR^d$ with
    $\hat{\vw}=\frac{\vw}{\norm{\vw}}$, we have
    \begin{align}
        \frac{1}{\norm{\vw}^d}\delta_{\hat{\vw}}\left(
        \frac{\vx}{\norm{\vw}}\right)=\delta_{\vw}(\vx).
    \end{align}
\end{lem}
\begin{proof} This is proved by changing of variables. In fact, for any $\phi\in\fS(\sR^d)$, we have
    \begin{align*}
        \left\langle\frac{1}{\norm{\vw}^d}\delta_{\hat{\vw}}\left(
        \frac{\cdot}{\norm{\vw}}\right),\phi(\cdot)\right\rangle_{\fS'(\sR^d),\fS(\sR^d)}
         & =\left\langle\delta_{\hat{\vw}}(\cdot), \phi(\norm{\vw}\cdot)\right\rangle_{\fS'(\sR^d),\fS(\sR^d)}\\
         & = \int_{\sR}\phi\left(\norm{\vw}y\hat{\vw}\right)\diff{y} \\
         & = \int_{\sR}\phi(y\vw)\diff{y}                                         \\
         & = \left\langle\delta_{\vw}(\cdot),\phi(\cdot)\right\rangle_{\fS'(\sR^d),\fS(\sR^d)}.
    \end{align*}
\end{proof}
\begin{lem}[Fourier transforms of network functions]
    For any unit vector $\vnu\in\sR^d$, any nonzero vector $\vw\in\sR^d$ with
    $\hat{\vw}=\frac{\vw}{\norm{\vw}}$, and $g\in\fS'(\sR)$ with
    $\fF[g]\in C(\sR)$, we have, in the sense of distribution,
    \begin{align}
        \text{(a)}\quad & \fF_{\vx\to\vxi}[g(\vnu^\T\vx)](\vxi)
          = \delta_{\vnu}(\vxi)\fF[g](\vxi^\T\vnu),                          \\
        \text{(b)}\quad & \fF_{\vx\to\vxi}[g(\vw^\T\vx+b)](\vxi)
          = \delta_{\vw}(\vxi)\fF[g]\left(
        \frac{\vxi^\T\hat{\vw}}{\norm{\vw}}\right)\E^{2\pi\I
            \frac{b}{\norm{\vw}}\vxi^\T\hat{\vw}},                            \\
        \text{(c)}\quad & \fF_{\vx\to\vxi}[\vx g(\vw^\T\vx+b)](\vxi)
          = \frac{\I}{2\pi}\nabla_{\vxi}\left[\delta_{\vw}(\vxi)\fF[g]\left(
            \frac{\vxi^\T\hat{\vw}}{\norm{\vw}}\right)\E^{2\pi\I \frac{b}{\norm{\vw}}
                \vxi^\T\hat{\vw}}\right].
    \end{align}
\end{lem}
\begin{proof}
    Let $\phi\in\fS(\sR^d)$ be any test function.
    \begin{enumerate}[(a)] 
        \item By direct calculation, we have
              \begin{align*}
                 \left\langle\fF_{\vx\to\cdot}[g(\vnu^\T\vx)](\cdot),\phi(\cdot)\right\rangle_{\fS'(\sR^d),\fS(\sR^d)}
                   & = \left\langle g(\vnu^\T\cdot),\fF_{\vx\to\cdot}[\phi(\vx)](\cdot)\right\rangle_{\fS'(\sR^d),\fS(\sR^d)}\\
                   & = \left\langle g(\cdot),\fF_{y\to\cdot}[\phi(y\vnu)](\cdot)\right\rangle_{\fS'(\sR),\fS(\sR)}\\
                   & = \left\langle\fF_{y\to\cdot}[g(y)](\cdot),\phi(\cdot\vnu)\right\rangle_{\fS'(\sR),\fS(\sR)}\\
                   & = \left\langle\fF[g](\cdot\vnu^\T\vnu),\phi(\cdot\vnu)\right\rangle_{\fS'(\sR),\fS(\sR)}\\
                   & = \left\langle\delta_{\vnu}(\cdot)\fF[g](\cdot^\T\vnu)
                  ,\phi(\cdot)\right\rangle_{\fS'(\sR^d),\fS(\sR^d)}.
              \end{align*}
        \item By part (a), we have in the distributional sense
              \[
                  \fF_{\vx\to\vxi}[g(\hat{\vw}^\T\vx)](\vxi)=\delta_{\hat{\vw}}(\vxi)
                  \fF[g](\vxi^\T\hat{\vw}).
              \]
              Note that
              \[
                  \fF_{\vx\to\vxi}[g(\vx-\vx_0)](\vxi)
                  =\fF_{\vx\to\vxi}[g](\vxi)\E^{-2\pi\I\vx_0^\T\vxi},
              \]
              then
              \begin{align*}
                  \fF_{\vx\to\vxi}[g(\hat{\vw}^\T \vx+b)](\vxi)
                   & = \fF_{\vx\to\vxi}[g(\hat{\vw}^\T(\vx+b\hat{\vw}))](\vxi) \\
                   & = \delta_{\hat{\vw}}(\vxi)\fF[g](\vxi^\T\hat{\vw})
                  \E^{2\pi\I b\hat{\vw}^\T\vxi}.
              \end{align*}
              Therefore
              \begin{align*}
                  \fF_{\vx\to\vxi}[g(\vw^\T\vx+b)](\vxi)
                   & = \fF_{\vx\to\vxi}[g(\hat{\vw}^\T\norm{\vw}\vx+b)](\vxi)       \\
                   & = \frac{1}{\norm{\vw}^d}\fF_{\vx\to\vxi}[g(\hat{\vw}^\T\vx+b)]
                  \left(\frac{\vxi}{\norm{\vw}}\right)                              \\
                   & = \frac{1}{\norm{\vw}^d}\delta_{\hat{\vw}}\left(
                  \frac{\vxi}{\norm{\vw}}\right)\fF[g]\left(
                  \frac{\vxi^\T\hat{\vw}}{\norm{\vw}}\right)\E^{2\pi\I
                      \frac{b}{\norm{\vw}}\hat{\vw}^\T\vxi}                         \\
                   & = \delta_{\vw}(\vxi)\fF[g]\left(
                  \frac{\vxi^\T\hat{\vw}}{\norm{\vw}}\right)\E^{2\pi\I
                      \frac{b}{\norm{\vw}}\hat{\vw}^\T\vxi}.
              \end{align*}
        \item This follows from part (b) and the fact that for any function
              $\tilde{g}(\vx)$
              \begin{equation*}
                  \fF_{\vx\to\vxi}[\vx \tilde{g}(\vx)](\vxi)
                  = \frac{\I}{2\pi}\nabla_{\vxi}\left[\fF[g](\vxi)\right].
              \end{equation*}
    \end{enumerate}
\end{proof}


\section{Exact derivation of LFP model}\label{sec:lfpmodel}
In this section, we first present the general form of the LFP model for two-layer neural networks. Then, we exactly compute the LFP model in the Fourier domain and derive the expressions for two commonly-used activation functions, i.e., $\ReLU(x):=\max(x,0)$ and $\tanh(x)$. 

For any positive integer $N$, we denote the set $\{1,2,\cdots,N\}$ by $[N]$. The training data-set $S=\{(\vx_i,y_i)\})_{i=1}^n$, where $\{\vx_i\}_{i=1}^n$ are i.i.d. sampled from unknown distribution $\fD$ on a domain $\Omega\subset\sR^d$ and $y_i=f(\vx_i)$, $i\in[n]$ for some unknown function $f$.

\subsection{Mean-field kernel dynamics in frequency domain}


We suppose that $f\in C(\sR^d)\cap L^2(\sR^d)$ and that the activation function is locally $H^1$ and grows polynomially, i.e., $\abs{\sigma(z)}\leq C\abs{z}^p$ for some $p>0$. 


We consider the following gradient descent dynamics of the population risk $\RS$ of a network function $f(\cdot,\vtheta)$ parameterized by $\vtheta$
\begin{equation}
    \left\{
        \begin{array}{l}
             \dot{\vtheta}=-\nabla_{\vtheta}\RS(\vtheta),  \\
             \vtheta(0)=\vtheta_0,
        \end{array}
    \right.
\end{equation}
where
\begin{equation}
    \RS(\vtheta)
    = \frac{1}{2}\sum_{i=1}^n(f(\vx_i,\vtheta)-y_i)^2.
\end{equation}
Then the training dynamics of output function $f(\cdot,\vtheta)$ is
\begin{align*}
    \frac{\D}{\D t}f(\vx,\vtheta)
    &= \nabla_{\vtheta}f(\vx,\vtheta)\cdot\dot{\vtheta}\\
    &= -\nabla_{\vtheta}f(\vx,\vtheta)\cdot\nabla_{\vtheta}\RS(\vtheta)\\
    &= -\nabla_{\vtheta}f(\vx,\vtheta)\cdot\sum_{i=1}^n \nabla_{\vtheta}f(\vx_i,\vtheta)(f(\vx_i,\vtheta)-y_i)\\
    &= -\sum_{i=1}^n K_m(\vx,\vx_i)(f(\vx_i,\vtheta)-y_i)
\end{align*}
where for time $t$ the NTK evaluated at $(\vx,\vx')\in\Omega\times\Omega$ reads as
\begin{equation}
    K_m(\vx,\vx')(t)=\nabla_{\vtheta}f(\vx,\vtheta(t))\cdot\nabla_{\vtheta}f(\vx',\vtheta(t)).
\end{equation}
The gradient descent of the linear model thus becomes
\begin{equation}
    \frac{\D}{\D t}\Big(f(\vx,\vtheta(t))-f(\vx)\Big)=-\sum_{i=1}^n K_m(\vx,\vx_i)(t)\Big(f(\vx_i,\vtheta(t))-f(\vx_i)\Big).
\end{equation}
Define the residual $\vu(\vx,t)=f(\vx,\vtheta(t))-f(\vx)$ and the empirical density $\rho(\vx)=\sum_{i=1}^n\delta(\vx-\vx_i)$. We further denote $u_{\rho}(\vx)=u(\vx)\rho(\vx)$. Therefore the dynamics for $u$ becomes
\begin{equation}
    \frac{\D}{\D t}u(\vx,t)=-\int_{\sR^d}K_m(\vx,\vx')(t)u_{\rho}(\vx',t)
    \diff{\vx'}.\label{eq..DynamicsFiniteWidth}
\end{equation}
From now on, we consider the two-layer neural network
\begin{align}
    f(\vx,\vtheta)
    &= \frac{1}{\sqrt{m}}\sum_{j=1}^{m}a_{j}\sigma(\vw_{j}^{\T}\vx+b_{j})\label{eq: 2layer-nn}\\
    &= \frac{1}{\sqrt{m}}\sum_{j=1}^{m}\sigma^{*}(\vx,\vq_{j}).
\end{align}
where the vector of all parameters  $\vtheta=\mathrm{vec}(\{\vq_j\}_{j=1}^m)$ is formed of the parameters for each neuron $\vq_{j}={(a_{j},\vw_{j}^{\T},b_{j})}^{\T}\in\sR^{d+2}$
and $\sigma^{*}(\vx,\vq_{j})=a_{j}\sigma(\vw_{j}^{\T}\vx+b_{j})$
for $j\in[m]$. We consider the kernel regime that $m\gg 1$ and assume that $b\sim\fN(0,\sigma_{b}^{2})$ with $\sigma_{b}\gg 1$. 
For the two-layer network, its NTK can be calculated as follows
\begin{equation}
    K_m(\vx,\vx')(t)= \frac{1}{m}\sum_{j=1}^m\nabla_{\vq_j}\sigma^*(\vx,\vq_j(t))\cdot\sigma^*(\vx',\vq_j(t)),
\end{equation}
where the parameters $\vq_j$'s are evaluated at time $t$.
Under some weak condition and for sufficiently large $m$, E et al. \citet{e2019comparative} proved that the dynamics \eqref{eq..DynamicsFiniteWidth}, with a high probability, converges to the following dynamics for any $t\in\sR$
\begin{equation}
    \frac{\D}{\D t}u(\vx,t)=-\int_{\sR^d}K(\vx,\vx')u_{\rho}(\vx',t)
    \diff{\vx'}.\label{eq..DynamicsInfiniteWidth}
\end{equation}
where the kernel only depends on the initial distribution of parameters and reads as
\begin{align}
    K(\vx,\vx')
    &= \Exp_{\vq} \nabla_{\vq}\sigma^*(\vx,\vq)\cdot\sigma^*(\vx',\vq)\\
    &= \Exp_{\vq} (\sigma(\vw^\T\vx+b)\sigma(\vw^\T\vx'+b)+a^2\sigma'(\vw^\T\vx+b)\sigma'(\vw^\T\vx'+b)\vx^\T\vx' \nonumber\\
    &~~+a^2\sigma'(\vw^\T\vx+b)\sigma'(\vw^\T\vx'+b)).
\end{align}
Intuitively, this is because $K_m(\vx,\vx')(t)=K(\vx,\vx')+O(\frac{1}{\sqrt{m}})$ according to the law of large numbers. 
In the following, we analyze \eqref{eq..DynamicsInfiniteWidth} and calculate its formulation in the frequency domain.

We start with the following lemma.
\begin{lem}[LFP dynamics for general DNNs]\label{lem:dynamics}
    The dynamics \eqref{eq..DynamicsInfiniteWidth} has the following expression in the frequency domain
    \begin{equation}
        \langle\partial_{t}\fF[u],\phi\rangle = \langle\fL[\fF[u_{\rho}]],\phi\rangle, \label{lfpprocess}
    \end{equation}
    where $\fL[\cdot]$ is called Linear F-Principle (LFP) operator is given by
    \begin{equation*}
        \fL[\fF[u_{\rho}]]=-\int_{\sR^{d}}\hat{K}(\vxi,\vxi')\fF[u_{\rho}](\vxi')\diff{\vxi'},
    \end{equation*}
    and 
    \begin{equation}
        \hat{K}(\vxi,\vxi')  :=\Exp_{\vq}\hat{K}_{\vq}(\vxi,\vxi') :=\Exp_{\vq}\fF_{\vx\to\vxi}[\nabla_{\vq}\sigma^{*}(\vx,\vq)]\cdot\overline{\fF_{\vx'\to\vxi'}[\nabla_{\vq}\sigma^{*}(\vx',\vq)]}.
    \end{equation}
    The expectation $\Exp_{\vq}$ is taken w.r.t. initial distribution of parameters.
\end{lem}
\begin{proof}
    
    For any $\phi\in \fS(\sR^{d})$. since $\partial_t u$ is in $\fS'(\sR^d)$ and locally integrable, we have
    
    \begin{align*}
        \langle\partial_{t}\fF[u],\phi\rangle
        &= \langle\partial_{t}u,\fF[\phi]\rangle\\
        &= \int_{\sR^{d}}\partial_{t}u(\vx,t)\int_{\sR^{d}}\phi(\vxi)\E^{-\I2\pi\vx\cdot\vxi}\diff{\vxi}\diff{\vx}\\
        &= -\int_{\sR^{d}}\int_{\sR^{d}}K(\vx,\vx')u_{\rho}(\vx')\diff{\vx'} \int_{\sR^{d}}\phi(\vxi)\E^{-\I2\pi\vx\cdot\vxi}\diff{\vxi}\diff{\vx}\\
        &= -\int_{\sR^{3d}}K(\vx,\vx')u_{\rho}(\vx')\diff{\vx'}\phi(\vxi)\E^{-\I2\pi\vx\cdot\vxi}\diff{\vxi}\diff{\vx}\\
        &= -\int_{\sR^{3d}}\Exp_{\vq}\nabla_{\vq}\sigma^*(\vx,\vq)\cdot\nabla_{\vq}\sigma^*(\vx',\vq)u_{\rho}(\vx')\diff{\vx'}\phi(\vxi)\E^{-\I2\pi\vx\cdot\vxi}\diff{\vxi}\diff{\vx}\\
        &= -\Exp_{\vq}\int_{\sR^{d}}\nabla_{\vq}\sigma^*(\vx',\vq)u_{\rho}(\vx')\diff{\vx'}\cdot\int_{\sR^{2d}}\nabla_{\vq}\sigma^*(\vx,\vq)\E^{-\I2\pi\vx\cdot\vxi}\phi(\vxi)\diff{\vxi}\diff{\vx}\\
        &= -\Exp_{\vq}\int_{\sR^{d}}\nabla_{\vq}\sigma^*(\vx',\vq)u_{\rho}(\vx')\diff{\vx'}\cdot
        \left\langle\fF_{\vx\to\cdot}[\nabla_{\vq}\sigma^*(\vx,\vq)](\cdot),\phi(\cdot)\right\rangle.
    \end{align*}
    Since 
    \begin{equation*}
        \int_{\sR^{d}}\nabla_{\vq}\sigma^*(\vx',\vq)u_{\rho}(\vx')\diff{\vx'}
        =\int_{\sR^{d}}\overline{\fF_{\vx'\to\vxi'}[\nabla_{\vq}\sigma^*(\vx',\vq)](\vxi')}\fF_{\vx'\to\vxi'}[u_{\rho}](\vxi')\diff{\vxi'},
    \end{equation*}
    we have
    \begin{align*}
        \langle\partial_{t}\fF[u],\phi\rangle
        &= -\Exp_{\vq}\int_{\sR^{d}}\overline{\fF_{\vx'\to\vxi'}[\nabla_{\vq}\sigma^*(\vx',\vq)](\vxi')}\fF_{\vx'\to\vxi'}[u_{\rho}](\vxi')\diff{\vxi'}\cdot\left\langle\fF_{\vx\to\cdot}[\nabla_{\vq}\sigma^*(\vx,\vq)](\cdot),\phi(\cdot)\right\rangle\\
        &= -\Exp_{\vq}\int_{\sR^{2d}}\overline{\fF_{\vx'\to\vxi'}[\nabla_{\vq}\sigma^*(\vx',\vq)](\vxi')}\cdot\fF_{\vx\to\vxi}[\nabla_{\vq}\sigma^*(\vx,\vq)](\vxi)\fF_{\vx'\to\vxi'}[u_{\rho}](\vxi')\diff{\vxi'}\phi(\vxi)\diff{\vxi}\\
        &= -\int_{\sR^{2d}}\hat{K}(\vxi,\vxi')\fF[u_\rho](\vxi')\diff{\vxi'}\phi(\vxi)\diff{\vxi}\\
        &= \langle\fL[\fF[u_{\rho}]],\phi\rangle.
    \end{align*}
\end{proof}

\subsection{LFP dynamics derived for two-layer networks}
In this section, we derive the LFP dynamics for two-layer networks with general activation function. The key difficulty comes from the repeated integral representation of the operator. By using the Laplace method in a proper way, we overcome this difficulty and arrive at a simpler expression for the dynamics.

To simplified the notation, we define $\vg_1(z):=(\sigma(z),a\sigma'(z))^\T$ and $g_2(z):=a\sigma'(z)$ for $z\in\sR$. Then
\begin{align}
    \vg_1(\vw^\T\vx+b)
    &= \begin{pmatrix}
        \sigma(\vw^\T\vx+b)\\
        a\sigma'(\vw^\T\vx+b)
    \end{pmatrix}
    = \begin{pmatrix}
        \partial_a [a\sigma(\vw^\T\vx+b)] \\
        \partial_b [a\sigma(\vw^\T\vx+b)]
    \end{pmatrix}
    ,\\
    g_2(\vw^\T\vx+b)\vx
    &= \nabla_{\vw}[a\sigma(\vw^\T\vx+b)]
    = a\sigma'(\vw^\T\vx+b)\vx.
\end{align}

The following theorem is the key to the exact expression of LFP dynamics for two-layer networks.
\begin{assump}\label{assump..InitialDist}
    We assume that the initial distribution of $\vq=(a,\vw^\T,b)^\T$ satisfies the following conditions:
    \begin{enumerate}[(i)]
        \item independence of $a,\vw,b$: $\rho_{\vq}(\vq)=\rho_{a}(a)\rho_{\vw}(\vw)\rho_{b}(b)$.
        \item zero-mean and finite variance of $b$: $\Exp_{b}b=0$ and $\Exp_b b^2=\sigma_b^2<\infty$.
        \item radially symmetry of $\vw$
        : $\rho_{\vw}(\vw)=\rho_{\vw}(\norm{\vw}\ve_1)$ where $\ve_1=(1,0,\cdots,0)^\T$.
    \end{enumerate}
\end{assump}
\begin{thm}[Main result: explicit expression of LFP operator for two-layer networks]\label{mainthm}
    Suppose that Assumption \ref{assump..InitialDist} holds. If $\sigma_b\gg 1$, then the dynamics \eqref{eq..DynamicsInfiniteWidth} has the following expression,
    \begin{equation} \label{thmdyna}
        \langle\partial_t\fF[u], \phi\rangle = -\left\langle \fL[\fF[u_{\rho}]], \phi \right\rangle+O(\sigma_b^{-3}),
    \end{equation}
    where $\phi\in \fS(\sR^d)$ is a test function and the LFP operator is given by
    \begin{equation} \label{eq..lfpoperatorthm}
        \begin{aligned}
            \fL[\fF[u_{\rho}]]
             & = \frac{\Gamma(d/2)}{2\sqrt{2}\pi^{(d+1)/2}\sigma_b\norm{\vxi}^{d-1}}\Exp_{a,r}\left[\frac{1}{r}\fF[\vg_1]\left(\frac{\norm{\vxi}}{r}\right)\cdot\fF[\vg_1]\left(\frac{-\norm{\vxi}}{r}\right)\right]\fF[u_{\rho}](\vxi)                                     \\
             & \quad -\frac{\Gamma(d/2)}{2\sqrt{2}\pi^{(d+1)/2}\sigma_b}\nabla\cdot\left (\Exp_{a,r}\left[\frac{1}{r\norm{\vxi}^{d-1}}\fF[g_2]\left(\frac{\norm{\vxi}}{r}\right)\fF[g_2]\left(-\frac{\norm{\vxi}}{r}\right)\right]\nabla\fF[u_{\rho}](\vxi) \right).
        \end{aligned}
    \end{equation}
    The expectations are taken w.r.t. initial parameter distribution. Here $r = \norm{\vw}$ with the probability density 
    $\rho_r(r) := \frac{2\pi^{d/2}}{\Gamma(d/2)}\rho_{\vw}(r\ve_1)r^{d-1}$, $\ve_1=(1,0,\cdots,0)^\T$.
\end{thm}
\begin{rmk}
    The operator $\fL$ presents a unified framework for general activation functions.
\end{rmk}
\begin{rmk}
    The derivatives of most activation functions decay in the Fourier domain, e.g., $\ReLU$, $\tanh$, and sigmoid. Hence, the dynamics in \eqref{thmdyna} for higher frequency component is slower, i.e., F-Principle.
\end{rmk}
\begin{rmk}
    The last term in Eq. \eqref{eq..lfpoperatorthm} arises from the evolution of $\vw$ is much more complicated, without which our experiments show that the LFP model can still predict the learning results of two-layer wide NNs. 
\end{rmk}
\begin{proof}
    For simplicity, we assume that $b\sim\fN(0,\sigma^2_b)$, $\sigma_b\gg 1$ in this proof. It is straightforward to extend the proof to general distributions for $b$ as long as it is zero-mean and with variance $\sigma_b\gg 1$.
    
    1. Divide into two parts. Note that
    \begin{equation}
        \begin{pmatrix}
            \vg_1(\vw^\T\vx+b) \\
            \vx g_2(\vw^\T\vx+b)
        \end{pmatrix}
        = \begin{pmatrix}
            \partial_a [a\sigma(\vw^\T\vx+b)] \\
            \partial_b [a\sigma(\vw^\T\vx+b)] \\
            \nabla_{\vw} [a\sigma(\vw^\T\vx+b)]
        \end{pmatrix}
        =\nabla_{\vq}\sigma^*(\vx,\vq).
    \end{equation}
    One can split the Fourier transformed kernel $\hat{K}$ into two parts, more precisely, 
    \begin{equation*}
        \hat{K}=\Exp_{\vq}\hat{K}_{\vq},\quad \hat{K}_{\vq}=\hat{K}_{a,b}+\hat{K}_{
        \vw},
    \end{equation*}
    where
    \begin{align*}
         \hat{K}_{\vq}(\vxi,\vxi')
         & = \Exp_{\vq}\fF_{\vx\to\vxi}[\nabla_{\vq}\sigma^{*}(\vx,\vq)]\cdot\overline{\fF_{\vx'\to\vxi'}[\nabla_{\vq}\sigma^{*}(\vx',\vq)]},\\
        \hat{K}_{a,b}(\vxi,\vxi')
         & = \fF\left[\vg_1(\vw^\T\vx+b)\right]\cdot\overline{\fF
            \left[\vg_1(\vw^\T\vx'+b)\right]},\\
        \hat{K}_{\vw}(\vxi,\vxi')
         & = \fF\left[\vx g_2(\vw^\T\vx+b)\right]\cdot\overline{\fF\left[\vx g_2(\vw^\T\vx'+b)\right]}.
    \end{align*}
    For any $\phi,\psi\in \fS(\sR^d)$, we have
    \begin{align}
        \langle\hat{K}_{\vq}, \phi\otimes\psi\rangle
        :=
        \langle\hat{K}_{\vq}, \phi\otimes\psi\rangle_{\fS'(\sR^{2d}),\fS(\sR^{2d})}
        =
        \int_{\sR^{2d}}\hat{K}_{\vq}(\vxi,\vxi')\phi(\vxi)
        \psi(\vxi')\diff{\vxi}\diff{\vxi'}.
    \end{align}
    The expressions for $\hat{K}_{a,b}$ and $\hat{K}_{\vw}$ are similar.
    
    2. Calculate $\hat{K}_{a,b}(\vxi,\vxi')$.
    Since
    \begin{align*}
        \hat{K}_{a,b}(\vxi,\vxi')
         & = \delta_{\vw}(\vxi)\delta_{\vw}(\vxi')\fF[\vg_1]\left(
        \frac{\vxi^\T\hat{\vw}}{\norm{\vw}}\right)\cdot\overline{\fF[\vg_1]\left(
            \frac{\vxi'^\T\hat{\vw}}{\norm{\vw}}\right)}\E^{2\pi\I b{(\vxi-\vxi')}^\T
            \hat{\vw}/ \norm{\vw}},
    \end{align*}
    we have
    \begin{align*}
        \langle\hat{K}_{a,b}, \phi\otimes\psi\rangle
         & = \int_{\sR^{2d}}
        \delta_{\vw}(\vxi)\delta_{\vw}(\vxi')\fF[\vg_1]\left(
        \frac{\vxi^\T\hat{\vw}}{\norm{\vw}}\right)\cdot\overline{\fF[\vg_1]\left(
            \frac{\vxi'^\T\hat{\vw}}{\norm{\vw}}\right)}\E^{2\pi\I b{(\vxi-\vxi')}^\T
            \hat{\vw}/ \norm{\vw}}
        \phi(\vxi)\psi(\vxi')\diff{\vxi}\diff{\vxi'}                                  \\
         & = \int_{\sR\times\sR}\phi(\eta\vw)\psi(\eta'\vw)\fF[\vg_1](\eta)
        \cdot\overline{\fF[\vg_1](\eta')}
        \E^{2\pi\I b(\eta-\eta')}\diff{\eta}\diff{\eta'}.
    \end{align*}
    By assumption $b\sim\fN(0,\sigma_b^2)$, i.e.,
    $\rho_{b}(b)=\dfrac{1}{\sqrt{2\pi}\sigma_{b}}
        \E^{-\frac{b^{2}}{2\sigma_{b}^{2}}}$, then
    $\fF[\rho_{b}](\eta)=\E^{-2\pi^2\sigma_{b}^{2}\eta^{2}}$.
    \begin{align*}
        \Exp_{b}\left(\E^{2\pi\I b(\eta-\eta')}\right)
         & = \int_{\sR}\frac{1}{\sqrt{2\pi}\sigma_{b}}
        \E^{-b^2/2\sigma_b^2}\E^{2\pi\I b(\eta-\eta')}\diff{b} \\
         & = \fF[\rho_b]\left(-(\eta-\eta')\right)             \\
         & = \E^{-2\pi^2\sigma_{b}^{2}{(\eta-\eta')}^2}.
    \end{align*}
    Therefore
    \begin{align*}
        \Exp_{b}\left[\langle\hat{K}_{a,b}, \phi\otimes\psi\rangle\right]
         & = \int_{\sR\times\sR}\phi(\eta\vw)\psi(\eta'\vw)\fF[\vg_1](\eta)
        \cdot\overline{\fF[\vg_1](\eta)}\Exp_{b}\left[
            \E^{2\pi\I b(\eta-\eta')}\right]\diff{\eta}\diff{\eta'}         \\
         & = \int_{\sR\times\sR}\phi(\eta\vw)\psi(\eta'\vw)\fF[\vg_1](\eta)
        \cdot\overline{\fF[\vg_1](\eta)}
        \E^{-2\pi^2\sigma_{b}^{2}{(\eta-\eta')}^2}\diff{\eta}\diff{\eta'}.
    \end{align*}
    Applying the Laplace method, we have
    \begin{align*}
        \Exp_{b}\left[\langle\hat{K}_{a,b}, \phi\otimes\psi\rangle\right]
         & = \int_{\sR}\phi(\eta\vw)\fF[\vg_1](\eta)\cdot
        \left[\int_{\sR}\psi(\eta'\vw)\overline{\fF[\vg_1](\eta)}\E^{-2\pi^2
                \sigma_{b}^{2}{(\eta-\eta')}^2}\diff{\eta'}\right]\diff{\eta}  \\
         & = \int_{\sR}\phi(\eta\vw)\fF[\vg_1](\eta)\cdot
        \left[\psi(\eta\vw)\overline{\fF[\vg_1](\eta)}\frac{1}{\sqrt{2\pi}\sigma_b}
            +O(\sigma_b^{-3})\right]\diff{\eta}                                \\
         & = \frac{1}{\sqrt{2\pi}\sigma_b}\int_{\sR}\phi(\eta\vw)\psi(\eta\vw)
        \fF[\vg_1](\eta)\cdot\overline{\fF[\vg_1](\eta)}\diff{\eta}+O(\sigma_b^{-3}).
    \end{align*}
    Next we consider the expectation with respect to $\vw$. Up to error of order
    $O(\sigma_c^{-3})$, we have
    \begin{align*}
        \Exp_{\vw,b}\left[\langle\hat{K}_{a,b}, \phi\otimes\psi\rangle\right]
         & = \Exp_{\vw}\left[\frac{1}{\sqrt{2\pi}\sigma_b}
            \int_{\sR}\phi(\eta\vw)\psi(\eta\vw)\fF[\vg_1](\eta)
            \cdot\overline{\fF[\vg_1](\eta)}\diff{\eta}\right] \\
         & = \int_{\sR^{d+1}}\frac{1}{\sqrt{2\pi}\sigma_b}
        \phi(\eta\vw)\psi(\eta\vw)\fF[\vg_1](\eta)
        \cdot\overline{\fF[\vg_1](\eta)}\rho_{\vw}(\vw)\diff{\vw}\diff{\eta}.
    \end{align*}
    Here we assume that $\rho_{\vw}$ is radially symmetric
    so $\rho_{\vw}(\vw)$ is a function of $r:=\norm{\vw}$ only. By using spherical coordinate system, we have
    \begin{align*}
        1
         & = \int_{\sR^d}\rho_{\vw}(\vw)\diff{\vw}            \\
         & = \int_{\sR^d}\rho_{\vw}(\norm{\vw}\ve_1)\diff{\vw}     \\
         & = \int_{\sR^+}\int_{\sS^{d-1}}\rho_{\vw}(r\ve_1)r^{d-1}
        \diff{\hat{\vw}}\diff{r}                              \\
         & = \int_{\sR^+}\rho_r(r)\diff{r},
    \end{align*}
    where $\hat{\vw}\in\sS^{d-1}$ and we define
    \begin{equation}
        \rho_r(r) := \int_{\sS^{d-1}}\rho_{\vw}(r\ve_1)r^{d-1}\diff{\hat{\vw}}
        = \frac{2\pi^{d/2}}{\Gamma(d/2)}\rho_{\vw}(r\ve_1)r^{d-1},
    \end{equation}
    where $\Gamma(\cdot)$ is the gamma function.
    Then we introduce the following change of variables,
    \begin{equation*}
        \begin{cases}
            \vzeta = \eta\vw, \\
            r = \norm{\vw},
        \end{cases}
    \end{equation*}
    whose the Jacobian determinant is
    \begin{equation*}
        \det\left(\frac{\partial(\vzeta, r)}{\partial(\vw, \eta)}\right) =
        \det\begin{bmatrix}
            \eta   & 0      & \cdots & 0      & w_1    \\
            0      & \eta   & \cdots & 0      & w_2    \\
            \vdots & \vdots & \ddots & \vdots & \vdots \\
            0      & 0      & \cdots & \eta   & w_d    \\
            w_1/r  & w_2/r  & \cdots & w_d/r  & 0
        \end{bmatrix}
        = -r\eta^{d-1}=-r{\left(\frac{\norm{\vzeta}}{r}\right)}^{d-1}.
    \end{equation*}
    Thus
    \begin{equation}
        \left \{
        \begin{aligned}
            \vw  & = \dfrac{r\vzeta}{\norm{\vzeta}} \\
            \eta & = \dfrac{\norm{\vzeta}}{r},
        \end{aligned}
        \right.
    \end{equation}
    and its Jacobian determinant is
    \begin{equation*}
        \det\left(\frac{\partial(\vw, \eta)}{\partial(\vzeta, r)}\right)
        = -\frac{r^{d-1}}{r\norm{\vzeta}^{d-1}}.
    \end{equation*}
    So one can obtain,
    \begin{align*}
        \Exp&_{\vw,b}\left[\langle\hat{K}_{a,b}, \phi\otimes\psi\rangle\right]
          = \int_{\sR^{d+1}}\frac{1}{\sqrt{2\pi}\sigma_b}\phi(\eta\vw)\psi(\eta\vw)
        \fF[\vg_1](\eta)\cdot\overline{\fF[\vg_1](\eta)}
        \rho_{\vw}(r\ve_1)\diff{\vw}\diff{\eta}                                           \\
         & = \int_{\sR^d\times\sR^+}\frac{1}{\sqrt{2\pi}\sigma_b}
        \phi(\vzeta)\psi(\vzeta)\fF[\vg_1]\left(\frac{\norm{\vzeta}}{r}\right)
        \cdot\overline{\fF[\vg_1]\left(\frac{\norm{\vzeta}}{r}\right)}
        \frac{r^{d-1}}{r\norm{\vzeta}^{d-1}}\rho_{\vw}(r\ve_1)\diff{\vzeta}\diff{r}       \\
         & = \int_{\sR^d\times\sR^+}\frac{1}{\sqrt{2\pi}\sigma_b}
        \phi(\vzeta)\psi(\vzeta)\fF[\vg_1]\left(\frac{\norm{\vzeta}}{r}\right)
        \cdot\overline{\fF[\vg_1]\left(\frac{\norm{\vzeta}}{r}\right)}
        \frac{1}{r\norm{\vzeta}^{d-1}}\left[\frac{\Gamma(d/2)}{2\pi^{d/2}}\rho_r(r)\right]
        \diff{\vzeta}\diff{r}                                                        \\
         & = \frac{\Gamma(d/2)}{2\sqrt{2}\pi^{(d+1)/2}\sigma_b}\int_{\sR^d}
        \phi(\vzeta)\int_{\sR^+}\left[\frac{1}{r\norm{\vzeta}^{d-1}}
            \fF[\vg_1]\left(\frac{\norm{\vzeta}}{r}\right)\cdot\overline{\fF[\vg_1]
                \left(\frac{\norm{\vzeta}}{r}\right)}\right]\psi(\vzeta)
        \rho_r(r)\diff{r}\diff{\vzeta},
    \end{align*}
    Therefore taking $\psi = \fF[u_{\rho}]$, we have
    \begin{align}
        \fL_{a,b}[\fF[u_{\rho}]]
         & = \frac{\Gamma(d/2)}{2\sqrt{2}\pi^{(d+1)/2}\sigma_b\norm{\vxi}^{d-1}}\Exp_{a,r}\left[\frac{1}{r}\fF[\vg_1]\left(\frac{\norm{\vxi}}{r}\right)\cdot\overline{\fF[\vg_1]\left(\frac{\norm{\vxi}}{r}\right)}\right]\fF[u_{\rho}](\vxi)\nonumber\\
         & = \frac{\Gamma(d/2)}{2\sqrt{2}\pi^{(d+1)/2}\sigma_b\norm{\vxi}^{d-1}}\Exp_{a,r}\left[\frac{1}{r}\fF[\vg_1]\left(\frac{\norm{\vxi}}{r}\right)\cdot\fF[\vg_1]\left(-\frac{\norm{\vxi}}{r}\right)\right]\fF[u_{\rho}](\vxi).\label{Lab}
    \end{align}
    
    3. Calculate $\hat{K}_{\vw}(\vxi,\vxi')$. Since
    \begin{align*}
        \hat{K}_{\vw}(\vxi,\vxi')
         & = \frac{1}{4\pi^2}\nabla_{\vxi}\left[\delta_{\vw}(\vxi)\fF[g_2]\left(
            \frac{\vxi^\T\hat{\vw}}{\norm{\vw}}\right)\E^{2\pi\I b\vxi^\T\hat{\vw}/\norm{\vw}}\right]\cdot\nabla_{\vxi'}\left[\delta_{\vw}(\vxi')\overline{\fF[g_2]\left(
                \frac{\vxi'^\T\hat{\vw}}{\norm{\vw}}\right)}\E^{-2\pi\I b\vxi'^\T\hat{\vw}/\norm{\vw}}\right],
    \end{align*}
    we have
    \begin{align*}
        ~~~~&\langle\hat{K}_{\vw}, \phi\otimes\psi\rangle\\
         & = \frac{1}{4 \pi^2}\int_{\sR^d}\phi(\vxi)\nabla_{\vxi}\left[\delta_{\vw}(\vxi)\fF[g_2]\left(
            \frac{\vxi^\T\hat{\vw}}{\norm{\vw}}\right)\E^{2\pi\I b\vxi^\T\hat{\vw}/\norm{\vw}}\right]\diff{\vxi}\nonumber\\
        &\quad ~~~\cdot\int_{\sR^d}\psi(\vxi')\nabla_{\vxi'}\left[\delta_{\vw}(\vxi')\overline{\fF[g_2]\left(
                \frac{\vxi'^\T\hat{\vw}}{\norm{\vw}}\right)}\E^{-2\pi\I b\vxi'^\T\hat{\vw}/\norm{\vw}}\right]\diff{\vxi'} \\
         & = \frac{1}{4 \pi^2}\int_{\sR^d}\nabla_{\vxi}\phi(\vxi)\delta_{\vw}(\vxi)\fF[g_2]\left(
        \frac{\vxi^\T\hat{\vw}}{\norm{\vw}}\right)\E^{2\pi\I b\vxi^\T\hat{\vw}/\norm{\vw}}\diff{\vxi}
        \nonumber\\
        &\quad ~~~\cdot\int_{\sR^d}\nabla_{\vxi'}\psi(\vxi')\delta_{\vw}(\vxi')\overline{\fF[g_2]\left(
            \frac{\vxi'^\T\hat{\vw}}{\norm{\vw}}\right)}\E^{-2\pi\I b\vxi'^\T\hat{\vw}/\norm{\vw}}\diff{\vxi'}            \\
         & = \int_{\sR\times\sR}\nabla\phi(\eta\vw)\cdot\nabla\psi(\eta'\vw)\fF[g_2](\eta)
        \cdot\overline{\fF[g_2](\eta')}
        \E^{2\pi\I b(\eta-\eta')}\diff{\eta}\diff{\eta'}.
    \end{align*}
    By the same computation as for $\Hat{K}_{a,b}(\vxi,\vxi')$, we can get
    \begin{align*}
         ~~~~&\Exp_{\vw,b}\left[\langle\hat{K}_{\vw}, \phi\otimes\psi\rangle\right]\nonumber\\
         & = \frac{\Gamma(d/2)}{2\sqrt{2}\pi^{(d+1)/2}\sigma_b}\int_{\sR^d}\nabla\phi(\vzeta)\cdot\int_{\sR^+}\left[\frac{1}{r\norm{\vzeta}^{d-1}}\fF[g_2]\left(\frac{\norm{\vzeta}}{r}\right)\cdot\overline{\fF[g_2]\left(\frac{\norm{\vzeta}}{r}\right)}\right]\nabla\psi(\vzeta)\rho_r(r)\diff{r}\diff{\vzeta}\\
         & = \frac{\Gamma(d/2)}{2\sqrt{2}\pi^{(d+1)/2}\sigma_b}\int_{\sR^d}\nabla\phi(\vzeta)\cdot\Exp_{a,r}\left[\frac{1}{r\norm{\vzeta}^{d-1}}\fF[g_2]\left(\frac{\norm{\vzeta}}{r}\right)\cdot\overline{\fF[g_2]\left(\frac{\norm{\vzeta}}{r}\right)}\right]\nabla\psi(\vzeta)\diff{\vzeta}\\
         & = -\frac{\Gamma(d/2)}{2\sqrt{2}\pi^{(d+1)/2}\sigma_b}\int_{\sR^d}\phi(\vzeta)\nabla\cdot\left \{\Exp_{a,r}\left[\frac{1}{r\norm{\vzeta}^{d-1}}\fF[g_2]\left(\frac{\norm{\vzeta}}{r}\right)\cdot\overline{\fF[g_2]\left(\frac{\norm{\vzeta}}{r}\right)}\right]\nabla\psi(\vzeta)\right \}\diff{\vzeta}.
    \end{align*}
    Thus taking $\psi(\vxi) = \fF[u_{\rho}](\vxi)$, we have
    \begin{align}
        \fL_{\vw}[\fF[u_{\rho}](\vxi)] = -\frac{\Gamma(d/2)}{2\sqrt{2}\pi^{(d+1)/2}\sigma_b}\nabla\cdot\left (\Exp_{a,r}\left[\frac{1}{r\norm{\vxi}^{d-1}}\fF[g_2]\left(\frac{\norm{\vxi}}{r}\right)\fF[g_2]\left(-\frac{\norm{\vxi}}{r}\right)\right]\nabla\fF[u_{\rho}](\vxi)\right) \label{Lw}.
    \end{align}
    Finally, one can plug \eqref{Lab} and \eqref{Lw} into \eqref{lfpprocess} and obtain the dynamics \eqref{thmdyna}.
\end{proof}

\subsection{Exact LFP model for common activation functions}
Based on \eqref{eq..lfpoperatorthm}, we derive the exact LFP dynamics for the cases where the activation function is ReLU and tanh.

\begin{cor}[LFP operator for ReLU activation function]
    Suppose that Assumption \ref{assump..InitialDist} holds. If $\sigma_b\gg 1$ and $\sigma=\ReLU$, then the dynamics \eqref{eq..DynamicsInfiniteWidth} has the following expression,
    \begin{equation} \label{thmdyna.ReLU}
        \langle\partial_t\fF[u], \phi\rangle = -\left\langle \fL[\fF[u_{\rho}]], \phi \right\rangle+O(\sigma_b^{-3}),
    \end{equation}
    where $\phi\in \fS(\sR^d)$ is a test function and the LFP operator reads as
    \begin{equation} \label{eq..lfpoperatorthm.ReLU}
        \begin{aligned}
            \fL[\fF[u_{\rho}]]
             & = \frac{\Gamma(d/2)}{2\sqrt{2}\pi^{(d+1)/2}\sigma_b}\Exp_{a,r}\left[\frac{r^3}{16\pi^4\norm{\vxi}^{d+3}} + \frac{a^2 r}{4\pi^2\norm{\vxi}^{d+1}}\right]\fF[u_{\rho}](\vxi)                                     \\
             & \quad -\frac{\Gamma(d/2)}{2\sqrt{2}\pi^{(d+1)/2}\sigma_b}\nabla\cdot\left (\Exp_{a,r}\left[\frac{a^2 r}{4\pi^2\norm{\vxi}^{d+1}}\right]\nabla\fF[u_{\rho}](\vxi) \right).
        \end{aligned}
    \end{equation}
    The expectations are taken w.r.t. initial parameter distribution. Here $r = \norm{\vw}$ with the probability density 
    $\rho_r(r) := \frac{2\pi^{d/2}}{\Gamma(d/2)}\rho_{\vw}(r\ve_1)r^{d-1}$, $\ve_1=(1,0,\cdots,0)^\T$.
\end{cor}
\begin{proof}
    Let 
    \begin{align}
        f_a(\vx)
         & :=
        \nabla_{a}\left[a\ReLU(\vw\cdot\vx+b)\right]=\ReLU(\vw\cdot\vx+b), \\
        g_a(z)
         & :=\ReLU(z),                                                     \\
        f_b(\vx)
         & :=
        \nabla_{b}\left[a\ReLU(\vw\cdot\vx+b)\right]=aH(\vw\cdot\vx+b),    \\
        g_b(z)
         & :=aH(z),
    \end{align}
    so $\vg_1(z) = {(g_a(z), g_b(z))}^\T$ and $g_2(z) = g_b(z)$. Then
    \begin{align}
        \fF[g_a](\xi)
         & = -\frac{1}{4\pi^2\xi^{2}}+\frac{\I}{4\pi}\delta'(\xi),     \\
        \fF[g_b](\xi)
         & = a\left[\frac{1}{\I2\pi\xi}+\frac{1}{2}\delta(\xi)\right],
    \end{align}
    By ignoring all $\delta(\xi)$ and $\delta'(\xi)$ related to only the trivial $\bm{0}$-frequency, we obtain
    \begin{align}
        \frac{1}{r}\fF[g_a]\left(\frac{\norm{\vxi}}{r}\right)\fF[g_a]\left(\frac{-\norm{\vxi}}{r}\right)
         & = \frac{r^3}{16\pi^4\norm{\vxi}^{4}},  \\
        \frac{1}{r}\fF[g_b]\left(\frac{\norm{\vxi}}{r}\right)\fF[g_b]\left(\frac{-\norm{\vxi}}{r}\right)
         & = \frac{a^2 r}{4\pi^2\norm{\vxi}^{2}}.
    \end{align}
    We then obtain \eqref{eq..lfpoperatorthm.ReLU} by plugging these into \eqref{eq..lfpoperatorthm}.
\end{proof}

\begin{cor}[LFP operator for tanh activation function]
    Suppose that Assumption \ref{assump..InitialDist} holds. If $\sigma_b\gg 1$ and $\sigma=\tanh$, then the dynamics \eqref{eq..DynamicsInfiniteWidth} has the following expression,
    \begin{equation} \label{thmdyna.tanh}
        \langle\partial_t\fF[u], \phi\rangle = -\left\langle \fL[\fF[u_{\rho}]], \phi \right\rangle+O(\sigma_b^{-3}),
    \end{equation}
    where $\phi\in \fS(\sR^d)$ is a test function and the LFP operator reads as
    \begin{equation} \label{eq..lfpoperatorthm.tanh}
        \begin{aligned}
            \fL[\fF[u_{\rho}]]
             & = \frac{\Gamma(d/2)}{2\sqrt{2}\pi^{(d+1)/2}\sigma_b\norm{\vxi}^{d-1}}\Exp_{a,r}\left[\frac{\pi^2}{r}\csch^2\left(\frac{\pi^2\norm{\vxi}}{r}\right) + \frac{4\pi^4a^2\norm{\vxi}^2}{r^3}\csch^2\left(\frac{\pi^2\norm{\vxi}}{r}\right)\right]\fF[u_{\rho}](\vxi)                                     \\
             & \quad -\frac{\Gamma(d/2)}{2\sqrt{2}\pi^{(d+1)/2}\sigma_b}\nabla\cdot\left (\Exp_{a,r}\left[\frac{4\pi^4a^2}{r^3\norm{\vxi}^{d-3}}\csch^2\left(\frac{\pi^2\norm{\vxi}}{r}\right)\right]\nabla\fF[u_{\rho}](\vxi) \right).
        \end{aligned}
    \end{equation}
    The expectations are taken w.r.t. initial parameter distribution. Here $r = \norm{\vw}$ with the probability density 
    $\rho_r(r) := \frac{2\pi^{d/2}}{\Gamma(d/2)}\rho_{\vw}(r\ve_1)r^{d-1}$, $\ve_1=(1,0,\cdots,0)^\T$.
\end{cor}
\begin{proof}
    Let
    \begin{align}
        f_a(\vx)
         & := \nabla_{a}\left[a\tanh(\vw\cdot\vx+b)\right]=\tanh(\vw\cdot\vx+b),    \\
        g_a(z)
         & := \tanh(z),                                                             \\
        f_b(\vx)
         & := \nabla_{b}\left[a\tanh(\vw\cdot\vx+b)\right]=a\sech^2(\vw\cdot\vx+b), \\
        g_b(z)
         & := a\sech^2(z),
    \end{align}
    so $\vg_1(z) = {(g_a(z), g_b(z))}^\T$ and $g_2(z) = g_b(z)$. Then
    \begin{align}
        \fF[g_a](\xi)
         & = -\I\pi\csch(\pi^2\xi),     \\
        \fF[g_b](\xi)
         & = 2\pi^2a\xi\csch(\pi^2\xi).
    \end{align}
    By ignoring all $\delta(\xi)$ and $\delta'(\xi)$ related to only the trivial $\vzero$-frequency, we obtain
    \begin{align}
        \frac{1}{r}\fF[g_a]\left(\frac{\norm{\vxi}}{r}\right)\fF[g_a]\left(\frac{-\norm{\vxi}}{r}\right)
         & = \frac{\pi^2}{r}\csch^2\left(\frac{\pi^2\norm{\vxi}}{r}\right),                    \\
        \frac{1}{r}\fF[g_b]\left(\frac{\norm{\vxi}}{r}\right)\fF[g_b]\left(\frac{-\norm{\vxi}}{r}\right)
         & = \frac{4\pi^4a^2\norm{\vxi}^2}{r^3}\csch^2\left(\frac{\pi^2\norm{\vxi}}{r}\right).
    \end{align}
    We then obtain \eqref{eq..lfpoperatorthm.tanh} by plugging these into \eqref{eq..lfpoperatorthm}.
\end{proof}

\section{Explicitizing the implicit bias of the F-Principle }\label{sec:Explicitizing-the-implicit}
In the following sections, we analyze a simplified LFP model with ReLU activation function in \eqref{eq..lfpoperatorthm.ReLU} as follows, 
\begin{equation}  \label{eq..lfpoperatorthm.simpleReLU}
    \partial_t\fF[u]=\Exp_{a,r}\left[\frac{r^3}{16\pi^4\norm{\vxi}^{d+3}} + \frac{a^2 r}{4\pi^2\norm{\vxi}^{d+1}}\right]\fF[u_{\rho}](\vxi).   
\end{equation}
We discard the last term in Eq. \eqref{eq..lfpoperatorthm.ReLU} arising from the evolution of $\vw$. The reasons are two folds. First, experiments show that  Eq. \eqref{eq..lfpoperatorthm.simpleReLU} is accurately enough to predict the wide two-layer NN output after training. Second, the last term in Eq. \eqref{eq..lfpoperatorthm.ReLU} is too complicated to analyze for now. 

In the LFP model, the solution is implicitly regularized by a decaying coefficient for different frequencies of $\fF[u]$ throughout the training. For a quantitative analysis of this solution, we explicitize such an implicit dynamical regularization by a constrained optimization problem as follows.

\subsection{An equivalent optimization problem to the gradient flow dynamics}

First, we present a general theorem that the long-time limit solution of a gradient flow dynamics is equivalent to the solution of a constrained optimization
problem. 

Let $H_1$ and $H_2$ be two separable Hilbert spaces and $\fP: H_1\rightarrow H_2$ is a bounded linear operator. Let $\fP^*: H_2\rightarrow H_1$ be the adjoint operator of $\fP$, defined by 
\begin{equation}
  \langle \fP \phi_1, \phi_2\rangle_{H_2}=\langle \phi_1, \fP^* \phi_2\rangle_{H_1},\quad\text{for all}\quad \phi_1\in H_1, \phi_2\in H_2.
\end{equation}

\begin{lem}\label{lem..spectrum.positive}
  Suppose that $H_1$ and $H_2$ are two separable Hilbert spaces and $\fP:H_1\rightarrow H_2$ and $\fP^*:H_2\rightarrow H_1$ is the adjoint of $\fP$. Then all eigenvalues of $\fP^*\fP$ and $\fP\fP^*$ are non-negative. Moreover, they have the same positive spectrum. If in particular, we assume that the operator $\fP\fP^*$ is surjective, then the operator $\fP\fP^*$ is invertible.
\end{lem}
\begin{proof}
  We consider the eigenvalue problem $\fP^*\fP \phi_1=\lambda \phi_1$. Taking inner product with $\phi_1$, we have $\langle \phi_1,\fP^*\fP \phi_1\rangle_{H_1}=\lambda\norm{\phi_1}^2_{H_1}$. Note that the left hand side is $\norm{\fP \phi_1}^2_{H_2}$ which is non-negative. Thus $\lambda\geq 0$. Similarly, the eigenvalues of $\fP\fP^*$ are also non-negative.

  Now if $\fP^*\fP$ has a positive eigenvalue $\lambda>0$, then $\fP^*\fP \phi_1=\lambda \phi_1$ with non-zero vector $\phi_1\in H_1$. It follows that $\fP\fP^*(\fP \phi_1)=\lambda (\fP \phi_1)$. It is sufficient to prove that $\fP \phi_1$ is non-zero. Indeed, if $\fP \phi_1=0$, then $\fP^*\fP \phi_1=0$ and $\lambda=0$ which contradicts with our assumption. Therefore, any positive eigenvalue of $\fP^*\fP$ is an eigenvalue of $\fP\fP^*$. Similarly, any positive eigenvalue of $\fP\fP^*$ is an eigenvalue of $\fP^*\fP$.

  Next, suppose that $\fP\fP^*$ is surjective.
  We show that $\fP\fP^* \phi_2=0$ has only the trivial solution $\phi_2=0$. In fact, $\fP\fP^* \phi_2=0$ implies that $\norm{\fP^* \phi_2}^2_{H_1}=\langle \phi_2, \fP\fP^* \phi_2\rangle_{H_2}=0$, i.e., $\fP^*\phi_2=0$. Thanks to the surjectivity of $\fP\fP^*$, there exists a vector $\phi_3\in H_2$ such that $\phi_2=\fP\fP^* \phi_3$. Let $\phi_1=\fP^* \phi_3\in H_1$. Hence $\phi_2=\fP \phi_1$ and $\fP^*\fP h_1=0$. Taking inner product with $\phi_1$, we have $\norm{\fP \phi_1}^2_{H_2}=\langle \phi_1, \fP^*\fP \phi_1\rangle_{H_1}=0$, i.e., $\phi_2=\fP \phi_1=0$. Therefore $\fP\fP^*$ is injective. This with the surjectivity assumption of $\fP\fP^*$ leads to that $\fP\fP^*$ is invertible.
\end{proof}

\begin{rmk}
  For the finite dimensional case $H_2=\sR^n$, conditions for the operator $\fP$ in Lemma \ref{lem..spectrum.positive} are reduced to that the matrix $\mP$ has rank $n$ (full rank).
\end{rmk}

Given $g\in H_2$, we consider the following two problems.

(i)
The initial value problem
\begin{equation*}
  \left\{
    \begin{array}{ll}
      \dfrac{\D \phi}{\D t}&=\fP^*(g-\fP \phi)\\
      \phi(0)&=\phi_{\rm ini}.
    \end{array}
  \right.
\end{equation*}
Since this equation is linear and with nonpositive eigenvalues on the right hand side, there exists a unique global-in-time solution $\phi(t)$ for all $t\in[0,+\infty)$ satisfying the initial condition. Moreover, the long-time limit $\lim_{t\rightarrow+\infty}\phi(t)$ exists and will be denoted as $\phi_\infty$.

(ii)
The minimization problem
\begin{align*}
  &\min_{\phi -\phi_{\rm ini}\in H_1}\norm{\phi-\phi_{\rm ini}}_{H_1},\\
  &\text{s.t.}\quad \fP \phi=g.
\end{align*}
In the following, we will show it has a unique minimizer which is denoted as $h_{\min}$.

Now we show the following equivalent theorem.
\begin{thm}[Equivalence between gradient descent and optimization problems]  \label{thm..EquivalenceDynamicsMinimization}
  Suppose that $\fP\fP^*$ is surjective. The above Problems (i) and (ii) are equivalent in the sense that $\phi_\infty=\phi_{\min}$. 
  More precisely, we have
  \begin{equation}
    \phi_\infty=h_{\min}=\fP^*(\fP\fP^*)^{-1}(g-\fP \phi_{\rm ini})+\phi_{\rm ini}.
  \end{equation}
\end{thm}
\begin{proof}
  Let $\tilde{\phi}=\phi-\phi_{\rm ini}$ and $\tilde{g}=g-\fP \phi_{\rm ini}$. Then it is sufficient to show the following problems (i') and (ii') are equivalent.

  (i')
  The initial value problem
  \begin{equation*}
  \left\{
    \begin{array}{l}
      \dfrac{\D \tilde{\phi}}{\D t}
      = \fP^*(\tilde{g}-\fP\tilde{\phi})\\
      \tilde{\phi}(0)=0.
    \end{array}
  \right.
  \end{equation*}

  (ii')
  The minimization problem
  \begin{align*}
    &\min_{\tilde{\phi}}\norm{\tilde{\phi}}^2_{H_1},\\
    &\text{s.t.}\quad \fP\tilde{\phi}=\tilde{g}.
  \end{align*}

  We claim that $\tilde{\phi}_{\min}=\fP^*(\fP\fP^*)^{-1}\tilde{g}$. Thanks to Lemma \ref{lem..spectrum.positive}, $\fP\fP^*$ is invertible, and thus $\phi_{\min}$ is well-defined and satisfies that $\fP\tilde{\phi}=\tilde{g}$. It remains to show that this solution is unique. In fact, for any $\tilde{\phi}$ satisfying $\fP\tilde{\phi}=\tilde{g}$, we have
  \begin{align*}
    \langle\tilde{\phi}-\tilde{\phi}_{\min},\tilde{\phi}_{\min}\rangle_{H_1}
    &= \langle\tilde{\phi}-\tilde{\phi}_{\min},\fP^*(\fP\fP^*)^{-1}\tilde{g}\rangle_{H_1}\\
    &= \langle \fP(\tilde{\phi}-\tilde{\phi}_{\min}), (\fP\fP^*)^{-1}\tilde{g}\rangle_{H_2}\\
    &= \langle \fP\tilde{\phi}, (\fP\fP^*)^{-1}\tilde{g}\rangle_{H_2}-\langle \fP\tilde{\phi}_{\min}, (\fP\fP^*)^{-1}\tilde{g}\rangle_{H_2}\\
    &= 0.
  \end{align*}
  Therefore,
  \begin{equation*}
    \norm{\tilde{\phi}}^2_{H_1}=\norm{\tilde{\phi}_{\min}}^2_{H_1}+\norm{\tilde{\phi}-\tilde{\phi}_{\min}}^2_{H_1}\geq \norm{\tilde{\phi}_{\min}}^2_{H_1}.
  \end{equation*}
  The equality holds if and only if $\tilde{\phi}=\tilde{\phi}_{\min}$.

  For problem (i'), from the theory of ordinary differential equations on Hilbert spaces, we have that its solution can be written as
  \begin{equation*}
    \tilde{\phi}(t)=\fP^*(\fP\fP^*)^{-1}\tilde{g}+\sum_{i\in I}c_i v_i\exp(-\lambda_i t),
  \end{equation*}
  where $\lambda_i$, $i\in \fI$ are positive eigenvalues of $\fP\fP^*$, $\fI$ is an index set with at most countable cardinality, and $v_i$, $i\in \fI$ are eigenvectors in $H_1$.
  Thus $\tilde{\phi}_\infty=\tilde{\phi}_{\min}=\fP^*(\fP\fP^*)^{-1}\tilde{g}$.

  Finally, by back substitution, we have
  \begin{equation*}
    \phi_\infty=\phi_{\min}
    = \fP^*(\fP\fP^*)^{-1}\tilde{g}+\phi_0
    = \fP^*(\fP\fP^*)^{-1}(g-\fP \phi_{\rm ini})+\phi_{\rm ini}.
  \end{equation*}
\end{proof}

The following corollaries are obtained directly from Theorem \ref{thm..EquivalenceDynamicsMinimization}.

\begin{cor}\label{cor..EquivalencdTheta}
  Let $\phi$ be the parameter vector $\vtheta$ in $H_1=\sR^{m}$, $g$ be the outputs of the training data $\vY$, and $\mP$ be a full rank matrix in the linear DNN model. Then 
  the following two problems are equivalent in the sense that $\vtheta_\infty=\vtheta_{\min}$.

  (A1)
  The initial value problem
  \begin{equation*}
    \left\{
      \begin{array}{l}
        \dfrac{\D \vtheta}{\D t}=\mP^*(\vY-\mP\vtheta)\\
        \vtheta(0)=\vtheta_{\rm ini}.
      \end{array}
    \right.
  \end{equation*}

  (A2)
  The minimization problem
  \begin{align*}
    &\min_{\vtheta-\vtheta_{\rm ini}\in \sR^{m}}\norm{\vtheta-\vtheta_{\rm ini}}_{2},\\
    &\text{s.t.}\quad \mP\vtheta=\vY.
  \end{align*}
\end{cor}

The next corollary is a weighted version of Theorem \ref{thm..EquivalenceDynamicsMinimization}.
\begin{cor}\label{cor..EquivalencdHW}
  Let $H_1$ and $H_2$ be two separable Hilbert spaces and $\Gamma: H_1\rightarrow H_1$ be an injective operator.
  Define the Hilbert space $H_\Gamma:=\mathrm{Im}(\Gamma)$. 
  Let $g\in H_2$ and $\fP: H_\Gamma\rightarrow H_2$ be an operator such that $\fP\fP^*: H_2\to H_2$ is surjective.
  Then $\Gamma^{-1}: H_\Gamma\rightarrow H_1$ exists and $H_\Gamma$ is a Hilbert space with norm $\norm{\phi}_{H_\Gamma}:=\norm{\Gamma^{-1}\phi}_{H_1}$. Moreover, the following two problems are equivalent in the sense that $\phi_\infty=\phi_{\min}$.

  (B1)
  The initial value problem
  \begin{equation*}
    \left\{
      \begin{array}{l}
        \dfrac{\D \phi}{\D t}=\Gamma\Gamma^*\fP^*(g-\fP \phi)\\
        \phi(0)=\phi_{\rm ini}.
      \end{array}
    \right.
  \end{equation*}

  (B2)
  The minimization problem
  \begin{align*}
    &\min_{\phi-\phi_0\in H_\Gamma}\norm{\phi-\phi_{\rm ini}}_{H_\Gamma},\\
    &\text{s.t.}\quad \fP \phi=g.
  \end{align*}
\end{cor}
\begin{proof}
  The operator $\Gamma:H_1\rightarrow H_\Gamma$ is bijective. Hence $\Gamma^{-1}:H_\Gamma\rightarrow H_1$ is well-defined and $H_\Gamma$ with norm $\norm{\cdot}_{H_\Gamma}$ is a Hilbert space.
  The equivalence result holds by applying Theorem \ref{thm..EquivalenceDynamicsMinimization} with proper replacements. More precisely, we replace $\phi$ by $\Gamma^{-1} \phi$ and $\fP$ by $\fP\Gamma$. 
\end{proof}

\begin{cor}\label{cor..EquivalencdHWFrequency}
  Let $\gamma: \sR^{d}\rightarrow\sR^+$ be a positive function, $h$ be a function in $L^2(\sR^{d})$ and $\phi=\fF[h]$. The operator $\Gamma: L^2(\sR^{d})\rightarrow L^2(\sR^{d})$ is defined by $[\Gamma\phi](\vxi)=\gamma(\vxi)\phi(\vxi)$, $\vxi\in\sR^{d}$.
  Define the Hilbert space $H_\Gamma:=\mathrm{Im}(\Gamma)$. 
  Let $\mX=(\vx_1,\ldots,\vx_n)^\T\in \sR^{n\times d}$, $\vY=(y_1,\ldots,y_n)^\T \in \sR^{n}$ and $\fP: H_\Gamma\rightarrow \sR^{n}$ be a surjective operator
  \begin{equation}
  \fP: \phi\mapsto \left(\int_{\sR^{d}}\phi(\vxi)\E^{2\pi\I \vx_1^\T\vxi}\diff{\vxi},\ldots,\int_{\sR^{d}}\phi(\vxi)\E^{2\pi\I \vx_n^\T\vxi}\diff{\vxi}\right)^\T=(h(\vx_1),\ldots,h(\vx_n))^\T.
  \end{equation}
  Then the following two problems are equivalent in the sense that $\phi_\infty=\phi_{\min}$.

  (C1)
  The initial value problem
  \begin{equation*}
    \left\{
      \begin{array}{l}
        \displaystyle\dfrac{\D \phi(\vxi)}{\D t}=(\gamma(\vxi))^2\sum_{i=1}^n\left(y_i\E^{-2\pi\I \vx_i^\T\vxi}-\left[\phi*\E^{-2\pi\I \vx_i^\T(\cdot)}\right](\vxi)\right)\\
        \phi(0)=\phi_{\rm ini}.
      \end{array}
    \right.
  \end{equation*}

  (C2)
  The minimization problem
  \begin{align*}
    &\min_{\phi-\phi_{\rm ini}\in H_\Gamma}\int_{\sR^{d}}(\gamma(\vxi))^{-2}\abs{\phi(\vxi)-\phi_{\rm ini}(\vxi)}^2\diff{\vxi},\\
    &\text{s.t.}\quad h(\vx_i)=y_i,\quad i=1,\cdots,n.
  \end{align*}
\end{cor}
\begin{proof}
  Let $H_1=L^2(\sR^{d})$, $H_2=\sR^{n}$, $g=\vY$. By definition, $\Gamma$ is injective. Then by Corollary \ref{cor..EquivalencdHW}, we have that $\Gamma^{-1}: H_\Gamma\rightarrow L^2(\sR^{d})$ exists and $H_{\Gamma}$ is a Hilbert space with norm $\norm{\phi}_{H_\Gamma}:=\norm{\Gamma^{-1}\phi}_{L^2(\sR^{d})}$. Moreover, $\norm{\phi-\phi_{\rm ini}}_{H_\Gamma}^2=\int_{\sR^{d}}(\gamma(\vxi))^{-2}\abs{\phi(\vxi)-\phi_{\rm ini}(\vxi)}^2\diff{\vxi}$. We note that $[\fP^*Y](\vxi)=\sum_{i=1}^{n} y_i\E^{-2\pi\I \vx_i^\T\vxi}$ for all $\vxi\in\sR^{d}$. Thus
  \begin{align*}
    [\fP^*\fP\phi](\vxi)
    &= \left[\fP^*\left(\int_{\sR^{d}}\phi(\vxi')\E^{2\pi\I \vx_i^\T \vxi'}\diff{\vxi'}\right)_{i=1}^n\right](\vxi)\\
    &= \sum_{i=1}^n \int_{\sR^{d}}\phi(\vxi')\E^{2\pi\I \vx_i^\T\vxi'}\diff{\vxi'}\E^{-2\pi\I \vx_i^\T \vxi}\\
    &= \sum_{i=1}^n \int_{\sR^{d}}\phi(\vxi')\E^{-2\pi\I \vx_i^\T (\vxi-\vxi')}\diff{\vxi'}\\
    &= \sum_{i=1}^n \left[\phi*\E^{-2\pi\I \vx_i^\T(\cdot)}\right](\vxi).
  \end{align*}
  The equivalence result then follows from Corollary \ref{cor..EquivalencdHW}.
\end{proof}

We remark that $\fP^*\fP\phi=\sum_{i=1}^n\fF[h\delta_{\vx_i}]$, where $\delta_{\vx_i}(\cdot)=\delta(\cdot-\vx_i)$, $i=1,\cdots,n$. Therefore problem (C1) can also be written as:
  \begin{equation*}
    \left\{
      \begin{array}{l}
        \displaystyle\dfrac{\D \fF[h]}{\D t}=\gamma^2\sum_{i=1}^n(y_i\fF[\delta_{\vx_i}]-\fF[h\delta_{\vx_i}]
        )\\
        \fF[h](\vzero)=\fF[h]_{\rm ini}.
      \end{array}
    \right.
  \end{equation*}

In the following, we study the discretized version of this dynamics-optimization problem (C1\&C2).

\begin{cor}\label{cor..EquivalencdHWFrequencyDiscrete}
  Let $\gamma: \sZ^{d}\rightarrow\sR^+$ be a positive function defined on lattice $\sZ^{d}$ and $\phi=\fF[h]$. The operator $\Gamma: \ell^2(\sZ^{d})\rightarrow \ell^2(\sZ^{d})$ is defined by $[\Gamma\phi](\vk)=\gamma(\vk)\phi(\vk)$, $\vk\in\sZ^{d}$. Here $\ell^2(\sZ^{d})$ is set of square summable functions on the lattice $\sZ^{d}$.
  Define the Hilbert space $H_\Gamma:=\mathrm{Im}(\Gamma)$. 
  Let $X=(\vx_1,\ldots,\vx_n)^\T\in \sT^{n\times d}$, $Y=(y_1,\ldots,y_n)^\T \in \sR^{n}$ and $\fP: H_\Gamma\rightarrow \sR^{n}$ be a surjective operator such as 
  \begin{equation}
  P: \phi\mapsto \left(\sum_{\vk\in\sZ^{d}}\phi(\vk)\E^{2\pi\I \vx_1^\T\vk},\ldots,\sum_{\vk\in\sZ^{d}}\phi(\vk)\E^{2\pi\I \vx_n^\T\vk}\right)^\T.
  \end{equation}
  Then the following two problems are equivalent in the sense that $\phi_\infty=\phi_{\min}$.

  (D1)
  The initial value problem
  \begin{equation*}
    \left\{
      \begin{array}{ll}
        \displaystyle\dfrac{\D \phi(\vk)}{\D t}=(\gamma(\vk))^2\sum_{i=1}^n\left(y_i\E^{-2\pi\I \vx_i^\T \vk}-\left[\phi*\E^{-2\pi\I \vx_i^\T(\cdot)}\right](\vk)\right)\\
        \phi(\vzero)=\phi_{\rm ini}.
      \end{array}
    \right.
  \end{equation*}

  (D2)
  The minimization problem
  \begin{align*}
    &\min_{\phi-\phi_{\rm ini}\in H_\Gamma}\sum_{\vk\in\sZ^{d}}(\gamma(\vk))^{-2}\abs{\phi(\vk)-\phi_{\rm ini}(\vk)}^2,\\
    &\text{s.t.}\quad h(\vx_i)=y_i,\quad i=1,\cdots,n.
  \end{align*}
\end{cor}
\begin{proof}
  Let $H_1=\ell^2(\sZ^{d})$, $H_2=\sR^{n}$, and $g=\vY$. By definition, $\Gamma$ is injective. Then by Corollary \ref{cor..EquivalencdHW}, we have that $\Gamma^{-1}: H_\Gamma\rightarrow \ell^2(\sZ^{d})$ exists and $H_\Gamma$ is a Hilbert space with norm $\norm{\phi}_{H_\Gamma}:=\norm{\Gamma^{-1}\phi}_{\ell^2(\sZ^{d})}$. Moreover, $\norm{\phi-\phi_{\rm ini}}_{H_\Gamma}^2=\sum_{\vk\in\sZ^{d}}(\gamma(\vk))^{-2}\abs{\phi(\vk)-\phi_{\rm ini}(\vk)}^2$. 
  We note that $[P^*\vY](\vk)=\sum_{i=1}^{n} y_i\E^{-2\pi\I \vx_i^\T\vk}$ for all $\vk\in\sZ^{d}$. Thus
  \begin{align*}
    [P^*P\phi](\vk)
    &= \left[P^*\left(\sum_{\vk'\in\sZ^{d}}\phi(\vk')\E^{2\pi\I \vx_i^\T\vk'}\right)_{i=1}^n\right](\vk)\\
    &= \sum_{i=1}^n \sum_{\vk'\in\sZ^{d}}\phi(\vk')\E^{2\pi\I \vx_i^\T\vk'}\E^{-2\pi\I \vx_i^\T\vk}\\
    &= \sum_{i=1}^n \sum_{\vk'\in\sZ^{d}}\phi(\vk')\E^{-2\pi\I \vx_i^
    \T(\vk-\vk')}\\
    &= \sum_{i=1}^n \left[\phi*\E^{-2\pi\I \vx_i^\T(\cdot)}\right](\vk).
  \end{align*}
  The equivalence result then follows from Corollary \ref{cor..EquivalencdHW}.
\end{proof}

\subsection{Example: Explicitizing the implicit bias for two-layer ReLU NNs \label{subsec:Explicit-regularization-of}}

As an example, by Corollary \ref{cor..EquivalencdHWFrequency}, we derive
the following constrained optimization problem explicitly minimizing
an FP-norm (see next section), whose solution is the same as the long-time limit solution of the simplified LFP model \eqref{eq..lfpoperatorthm.simpleReLU},
that is, 
\begin{equation}
    \min_{h-h_{\mathrm{ini}}\in F_{\gamma}}\int_{\sR^d}\left(\Exp_{a,r}\left[\frac{r^3}{16\pi^4\norm{\vxi}^{d+3}} + \frac{a^2 r}{4\pi^2\norm{\vxi}^{d+1}}\right]\right)^{-1}\abs{\fF[h](\vxi)-\fF[h_{{\rm ini}}](\vxi)}^{2}\diff{\vxi},\label{eq: minFPnorm}
\end{equation}
subject to constraints
$h(\vx_{i})=y_{i}$ for $i=1,\cdots,n$. The $F_{\gamma}$ is defined in the next section.
This explicit penalty
indicates that the learning of DNN is biased towards functions with
more power at low frequencies,
which is speculated in previous works \citep{xu_training_2018,rahaman2018spectral,xu2019frequency}.  For 1-d problems ($d=1$), when $1/\xi^2$ term dominates, the corresponding minimization problem indicates a linear spline interpolation. Similarly, when $1/\xi^4$ dominates, the minimization problem indicates a cubic spline. In general, both two power law decays together lead to a specific mixture of linear and cubic splines. For high dimensional problems, the minimization problem is difficult to  be interpreted by a specific interpolation because the order of differentiation depends on $d$ and can be fractal.

\section{FP-norm and an \textit{a priori} generalization error bound \label{FPapriori}}

The equivalent explicit optimization problem \eqref{eq: minFPnorm}
provides a way to analyze the generalization of sufficiently wide two-layer NNs. We consider the Fourier domain with discretized frequencies. Then, we begin with the definition of
an FP-norm, which naturally induces a FP-space containing all possible solutions of a target NN, whose
 Rademacher complexity can be controlled by the FP-norm of the
target function. Thus we obtain an \textit{a priori} estimate of the generalization error of NN by the theory of Rademacher complexity. Our \textit{a priori}
estimates follows the Monte Carlo
error rates with respect to the sample size. Importantly, Our estimate unravels how frequency components of the target function affect the generalization performance of DNNs. 

\subsection{Problem Setup}

We focus on regression problem. Assume the target function $f:\Omega:=[0,1]^{d}\to\sR$.
Let the training set be $S=\{(\vx_{i}, y_{i})\}_{i=1}^{n}$,
where $\vx_{i}$'s are independently sampled from an underlying
distribution $\fD(\vx)$ and $y_{i}=f(\vx_{i})$. We consider
the square loss 
\begin{equation}
  \ell(h,\vx,y)=\abs{h(\vx)-y}^{2},
\end{equation}
with population risk 
\begin{equation}
  \RD(h)=\Exp_{\vx\sim\fD}\ell(h,\vx,f(\vx))
\end{equation}
and empirical risk 
\begin{equation}
  \RS(h)=\frac{1}{n}\sum_{i=1}^{n}\ell(h,\vx_{i},y_i).
\end{equation}

\subsection{FP-space}
The quantity in the minimization problem motives a definition of FP-norm, which would lead to the definition of the function space where the solution of the minimization problem lies in. 
We denote $\sZ^{d*}:=\sZ^d\backslash\{\vzero\}$.
Given a frequency weight function $\gamma: \sZ^{d}\to \sR^+$ or $\gamma: \sZ^{d*}\to \sR^+$ satisfying
\begin{equation}
  \norm{\gamma}_{\ell^2}=\left(\sum_{\vk\in\sZ^{d}}(\gamma(\vk))^2\right)^{\frac{1}{2}}<+\infty\quad\text{or}\quad \norm{\gamma}_{\ell^2}=\left(\sum_{\vk\in\sZ^{d*}}(\gamma(\vk))^2\right)^{\frac{1}{2}}<+\infty,
\end{equation}
we define the FP-norm for all function $h\in L^2(\Omega)$: 
\begin{equation}
  \norm{h}_{\gamma}:=\norm{\fF[h]}_{H_\Gamma}=\left(\sum_{\vk\in\sZ^{d}}(\gamma(\vk))^{-2}\abs{\fF[h](\vk)}^{2}\right)^{\frac{1}{2}}.
\end{equation}
If $\gamma:\sZ^{d*}\to\sR^+$ is not defined at $\vxi=\vzero$, we set $(\gamma(\vzero))^{-1}:=0$ in the above definition and $\norm{\cdot}_\gamma$ is only a semi-norm of $h$.

Then we define the FP-space
\begin{equation}
  \fF_{\gamma}(\Omega)=\{h\in L^2(\Omega):\norm{h}_{\gamma}<\infty\}.
\end{equation}
Clearly, for any $\gamma$, the FP-space is a subspace of $L^2(\Omega)$. In addition, if $\gamma: \vk\mapsto \norm{\vk}^{-r}$ for $\vk\in\sZ^{d*}$, then functions in the FP-space with $\fF[h](\vzero)=\int_{\Omega} h(\vx) \diff{\vx}=0$ form the Sobolev space $H^{r}(\Omega)$. Note that in the case of DNN, according to the F-Principle, $(\gamma(\vk))^{-2}$
increases with the frequency. Thus, the contribution of high frequency
to the FP-norm is more significant than that of low frequency.

\subsection{\textit{a priori} generalization error bound }
Next, we would show the upper bound of the FP-norm of a function leads to a upper bound of the Rademacher complexity of the function space. The Rademacher complexity is defined as 
\begin{equation}
    {\rm Rad}_{S}(\fH)=\frac{1}{n}\Exp_{\vtau}\left[\sup_{h\in\fH}\sum_{i=1}^{n}\tau_{i}h(\vx_{i})\right].
\end{equation}
for the function space $\fH$ and data-set $S=\{\vx_i,h(\vx_{i})\}_{i=1}^n$.

\begin{lem}\label{Rad bound}
  (i) For $\fH_Q=\{h:\norm{h}_{\gamma}\leq Q\}$ with $\gamma: \sZ^{d}\to \sR^+$, we have
  \begin{equation}
    {\rm Rad}_{S}(\fH_Q)\leq\frac{1}{\sqrt{n}}Q\norm{\gamma}_{\ell^2}.
  \end{equation}
  (ii) For $\fH_Q'=\{h:\norm{h}_{\gamma}\leq Q, \abs{\fF[h](\vzero)}\leq c_{0}\}$ with $\gamma: \sZ^{d*}\to \sR^+$ and $\gamma^{-1}(\vzero):=0$, we have
  \begin{equation}
    {\rm Rad}_{S}(\fH_Q')\leq \frac{c_{0}}{\sqrt{n}}+\frac{1}{\sqrt{n}}Q\norm{\gamma}_{\ell^2}.
  \end{equation}

\end{lem}

\begin{proof}
  We first prove (ii) since it is more involved.
  By the definition of the Rademacher complexity 
  \begin{equation}
    {\rm Rad}_{S}(\fH_Q')=\frac{1}{n}\Exp_{\vtau}\left[\sup_{h\in\fH_Q'}\sum_{i=1}^{n}\tau_{i}h(\vx_{i})\right].
  \end{equation}
  Let $\tau(\vx)=\sum_{i=1}^{n}\tau_{i}\delta(\vx-\vx_{i})$,
  where $\tau_{i}$'s are i.i.d. random variables with $\Prob(\tau_i=1)=\Prob(\tau_i=-1)=\frac{1}{2}$. We have $\fF[\tau](\vk)=\int_{\Omega}\sum_{i=1}^n\tau_i\delta(\vx-\vx_i)\E^{-2\pi\I \vk^\T \vx}\diff{\vx}=\sum_{i=1}^n\tau_i\E^{-2\pi\I \vk^\T\vx_i}$. Note that
  \begin{align}
    \sup_{h\in\fH_Q'}\sum_{i=1}^{n}\tau_{i}h(\vx_{i})
    = \sup_{h\in\fH_Q'}\sum_{i=1}^{n}\tau_{i}\bar{h}(\vx_{i})
    &= \sup_{h\in\fH_Q'}\sum_{i=1}^{n}\tau_{i}\sum_{\vk\in\sZ^d}\overline{\fF[h](\vk)}\E^{-2\pi\I \vk^\T\vx_i}\\
    &= \sup_{h\in\fH_Q'}\sum_{\vk\in\sZ^d}\fF[\tau](\vk)\overline{\fF[h](\vk)}.
  \end{align}
  By the Cauchy--Schwarz inequality,
  \begin{align}
        &~~\sup_{h\in\fH_Q'}\sum_{\vk\in\sZ^{d}}\fF[\tau](\vk)\overline{\fF[h](\vk)}\nonumber\\
        & \leq \sup_{h\in\fH_Q}\left[\fF[\tau](\vzero)\overline{\fF[h](\vzero)}+\left(\sum_{\vk\in\sZ^{d*}}(\gamma(\vk))^2\abs{\fF[\tau](\vk)}^{2}\right)^{1/2}\left(\sum_{\vk\in\sZ^{d*}}(\gamma(\vk))^{-2}\abs{\overline{\fF[h](\vk)}}^{2}\right)^{1/2}\right]\\
        & \leq c_{0}\abs{\fF[\tau](\vzero)}+Q\left(\sum_{\vk\in\sZ^{d*}}(\gamma(\vk))^2\abs{\fF[\tau](\vk)}^{2}\right)^{1/2}.
  \end{align}
  Since $\Exp_{\vtau}\abs{\fF[\tau](\vzero)}\leq (\Exp_{\vtau}\abs{\fF[\tau](\vzero)}^2)^{1/2}=\sqrt{n}$, $\Exp_{\vtau}\abs{\fF[\tau](\vk)}^{2}=\Exp_{\vtau}\sum_{i,j=1}^{n}\tau_{i}\tau_{j}\E^{-2\pi\I \vk^\T (\vx_{i}-\vx_{j})}=n$, we obtain
  \begin{align}
        \Exp_{\vtau}\left[\sup_{h\in\fH_Q'}\sum_{i=1}^{n}\tau_{i}h(\vx_{i})\right] 
        &\leq c_{0}\sqrt{n}+Q \Exp_{\vtau}\left(\sum_{\vk\in\sZ^{d*}}(\gamma(\vk))^2\abs{\fF[\tau](\vk)}^{2}\right)^{1/2}\\
        &\leq c_{0}\sqrt{n}+Q \left(\Exp_{\vtau}\sum_{\vk\in\sZ^{d*}}(\gamma(\vk))^2\abs{\fF[\tau](\vk)}^{2}\right)^{1/2}\\
        &= c_{0}\sqrt{n}+Q\sqrt{n}\norm{\gamma}_{\ell^2}.
  \end{align}
  This leads to 
  \begin{equation}
        {\rm Rad}_{S}(\fH_Q')\leq\frac{c_{0}}{\sqrt{n}}+\frac{1}{\sqrt{n}}Q\norm{\gamma}_{\ell^2}.
  \end{equation}

  For (ii), the proof is similar to (i). We have
  \begin{equation}
    \Exp_{\vtau}\left[\sup_{h\in\fH_Q}\sum_{\vk\in\sZ^{d}}\fF[\tau](\vk)\overline{\fF[h](\vk)}\right]
    \leq Q\Exp_{\vtau}\left(\sum_{\vk\in\sZ^{d}}(\gamma(\vk))^2|\fF[\tau](\vk)|^{2}\right)^{1/2}
    \leq Q\sqrt{n}\norm{\gamma}_{\ell^2}.
  \end{equation}
  Therefore
  \begin{equation}
    {\rm Rad}_{S}(\fH_Q)\leq \frac{1}{\sqrt{n}}Q\norm{\gamma}_{\ell^2}.
  \end{equation}
\end{proof}

Then, we prove that the target function can be used to bound the FP-norm of the solution of the minimization problem.
\begin{lem}\label{FPnorm Bound}
    Suppose that the real-valued target function $f\in \fF_\gamma(\Omega)$ and that the training dataset $\{(\vx_{i}, y_i)\}_{i=1}^{n}$ satisfies $y_i=f(\vx_{i})$, $i=1,\cdots,n$. If $\gamma: \sZ^{d}\to \sR^+$, then there exists a unique solution $h_{n}$ to the regularized model
  \begin{equation}
    \min_{h-h_{\rm ini}\in 
    \fF_\gamma(\Omega)} \norm{h-h_{{\rm ini}}}_{\gamma},\quad\text{s.t.}\quad h(\vx_i)=y_i,\quad i=1,\cdots,n.\label{eq..FPnormBoundMinimizationProblem}
  \end{equation}
  Moreover, we have
  \begin{equation}
    \norm{h_{n}-h_{{\rm ini}}}_\gamma\leq \norm{f-h_{{\rm ini}}}_\gamma.
  \end{equation}
\end{lem}

\begin{proof}
    By the definition of the FP-norm, we have $\norm{h_n-h_{\rm ini}}_\gamma=\norm{\fF[h]_n-\fF[h]_{\rm ini}}_{H_\Gamma}$. According to Corollary \ref{cor..EquivalencdHWFrequencyDiscrete}, the minimizer of problem \eqref{eq..FPnormBoundMinimizationProblem} exists, i.e., $h_n$ exists. Since the target function $f(x)$ satisfies the constraints $f(x_i)=y_i$, $i=1,\cdots,n$, we have $\norm{h_{n}-h_{{\rm ini}}}_\gamma\leq \norm{f-h_{{\rm ini}}}_\gamma$.
\end{proof}

\begin{lem}\label{0freqconstraint}
    Suppose that the real-valued target function $f\in \fF_\gamma(\Omega)$ and the training dataset $\{(\vx_{i}, y_i)\}_{i=1}^{n}$ satisfies $y_i=f(\vx_{i})$, $i=1,\cdots,n$. If $\gamma: \sZ^{d*}\to \sR^+$ with $\gamma^{-1}(\vzero):=0$, then there exists a solution $h_{n}$ to the regularized model
  \begin{equation}
    \min_{h-h_{\rm ini}\in \fF_
    \gamma(\Omega)} \norm{h-h_{{\rm ini}}}_{\gamma},\quad\text{s.t.}\quad h(\vx_i)=y_i,\quad i=1,\cdots,n.
  \end{equation}
  Moreover, we have 
  \begin{equation}
    \Abs{\fF[h_{n}-h_{\rm ini}](\vzero)}
    \leq \norm{f-h_{{\rm ini}}}_{\infty}+\norm{f-h_{{\rm ini}}}_{\gamma}\norm{\gamma}_{\ell^2}.
  \end{equation}
\end{lem}

\begin{proof}
  Let $f'=f-h_{\rm ini}$.
  Since $h_{n}(\vx_{i})-f(\vx_{i})=0$
  for $i=1,\cdots,n$, we have $h_{n}(\vx_{i})-f'(\vx_{i})-h_{{\rm ini}}(\vx_{i})=0$.
  Therefore
  \begin{align}
    \Abs{\fF[h_{n}-h_{\rm ini}](\vzero)}
    &= \Abs{f'(\vx_{i})-\sum_{\vk\in\sZ^{d*}}\fF[h_{n}-h_{\rm ini}](\vk)\E^{2\pi\I\vk^\T\vx_{i}}}\\
    &\leq \norm{f'}_{\infty}+\sum_{\vk\in\sZ^{d*}}\Abs{\fF[h_{n}-h_{\rm ini}](\vk)}\\
    &\leq \norm{f'}_{\infty}+\left(\sum_{\vk\in\sZ^{d*}}(\gamma(\vk))^2\right)^{\frac{1}{2}}\left(\sum_{\vk\in\sZ^{d*}}(\gamma(\vk))^{-2}\Abs{\fF[h_{n}-h_{\rm ini}](\vk)}^{2}\right)^{\frac{1}{2}}\\
    &\leq \norm{f'}_{\infty}+\norm{h_n-h_{\rm ini}}_{\gamma}\norm{\gamma}_{\ell^2}\\
    &\leq \norm{f'}_{\infty}+\norm{f'}_{\gamma}\norm{\gamma}_{\ell^2}.
  \end{align}
  We remark that the last step is due to the same reason as Lemma \ref{FPnorm Bound}.
\end{proof}
Based on above analysis, we derive an \textit{a priori} generalization error bound of the minimization problem. 
\begin{thm}[\textit{a priori} generalization error bound]\label{thm:priorierror}
  Suppose that the real-valued target function $f\in \fF_\gamma(\Omega)$, the training dataset $\{(\vx_{i}, y_i)\}_{i=1}^{n}$ satisfies $y_i=f(\vx_{i})$, $i=1,\cdots,n$, and $h_{n}$ is the solution of the regularized model
  \begin{equation}
    \min_{h-h_{\rm ini}\in \fF_
    \gamma(\Omega)} \norm{h-h_{{\rm ini}}}_{\gamma},\quad\text{s.t.}\quad h(\vx_i)=y_i,\quad i=1,\cdots,n.\label{eq:optf}
  \end{equation}
  Then we have

  (i) given $\gamma: \sZ^{d}\to \sR^+$, 
  for any $\delta\in(0,1)$, with probability at least $1-\delta$ over
  the random training sample, the population risk has the bound
  \begin{equation}
    \RD(h_{n})
    \leq \norm{f-h_{{\rm ini}}}_{\gamma}\norm{\gamma}_{\ell^2}
    \left(\frac{2}{\sqrt{n}}+4\sqrt{\frac{2\log(4/\delta)}{n}}\right).
  \end{equation}

  (ii) given $\gamma: \sZ^{d*}\to \sR^+$ with $\gamma(\vzero)^{-1}:=0$, for any $\delta\in(0,1)$,
  with probability at least $1-\delta$ over the random training sample,
  the population risk has the bound
  \begin{equation}
    \RD(h_{n})
    \leq \left(\norm{f-h_{\rm ini}}_{\infty}+2\norm{f-h_{{\rm ini}}}_{\gamma}\norm{\gamma}_{\ell^2}
    \right)
    \left(\frac{2}{\sqrt{n}}+4\sqrt{\frac{2\log(4/\delta)}{n}}\right).
  \end{equation}
\end{thm}

\begin{proof}
  Let $f'=f-h_{{\rm ini}}$ and $Q=\norm{f'}_{\gamma}$.

  (i) Given $\gamma:\sZ^d\to\sR^+$, we set $\fH_Q=\{h: \norm{h-h_{{\rm ini}}}_{\gamma}\leq Q\}$.
  According to Lemma \ref{FPnorm Bound}, the solution of problem \eqref{eq:optf} $h_{n}\in\fH_Q$. By the
  relation between generalization gap and Rademacher complexity \citep{bartlett2002rademacher,shalev2014understanding},
  \begin{equation}
  \abs{\RD(h_{n})-L_S(h_{n})}\leq 2{\rm Rad}_{S}(\fH_Q)+2\sup_{h,h'\in\fH_Q}\norm{h-h'}_{\infty}\sqrt{\frac{2\log(4/\delta)}{n}}.
  \end{equation}
  One of the component can be bounded as follows\textcolor{blue}{{} }
  \begin{align}
    \sup_{h,h'\in\fH_Q}\norm{h-h'}_{\infty} 
    & \leq \sup_{h\in\fH_Q}2\norm{h-h_{{\rm ini}}}_{\infty}\\
    & \leq \sup_{h\in\fH_Q}2\max_{\vx}\Abs{\sum_{\vk\in\sZ^{d}}\fF[h-h_{\rm ini}](\vk)\E^{2\pi\I \vk^\T\vx}}\\
    & \leq \sup_{h\in\fH_Q}2\sum_{\vk\in\sZ^{d}}\Abs{\fF[h-h_{\rm ini}](\vk)}\\
    & \leq 2\sup_{h\in\fH_Q}\left(\sum_{\vk\in\sZ^{d}}(\gamma(\vk))^2\right)^{\frac{1}{2}}\left(\sum_{\vk\in\sZ^{d}}(\gamma(\vk))^{-2}\Abs{\fF[h-h_{\rm ini}](\vk)}^{2}\right)^{\frac{1}{2}}\\
    & \leq 2Q\norm{\gamma}_{\ell^2}.
  \end{align}
  By Lemma \ref{Rad bound},
  \begin{equation}
    {\rm Rad}_{S}(\fH_Q)\leq \frac{1}{\sqrt{n}}Q\norm{\gamma}_{\ell^2}.
  \end{equation}
  By optimization problem \eqref{eq:optf}, $L_S(h_{n})\leq L_S(f')=0$.
  Therefore we obtain 
  \begin{equation}
    \RD(h)\leq \frac{2}{\sqrt{n}}\norm{f'}_{\gamma}\norm{\gamma}_{\ell^2}+4\norm{f'}_{\gamma}\norm{\gamma}_{\ell^2}\sqrt{\frac{2\log(4/\delta)}{n}}.
  \end{equation}

  (ii) Given $\gamma: \sZ^{d*}\to \sR^+$ with $\gamma(\vzero)^{-1}:=0$, set $c_0=\norm{f'}_{\infty}+\norm{f'}_{\gamma}\norm{\gamma}_{\ell^2}$. By Lemma \ref{Rad bound},
  \ref{FPnorm Bound}, and \ref{0freqconstraint}, define $\fH_Q'=\{h:\norm{h-h_{{\rm ini}}}_{\gamma}\leq Q,\abs{\fF[h-h_{\rm ini}](\vzero)}\leq c_0\}$,
  we obtain 
  \begin{equation}
    {\rm Rad}_{S}(\fH_Q')
    \leq \frac{1}{\sqrt{n}}\norm{f'}_{\infty}+\frac{2}{\sqrt{n}}\norm{f'}_{\gamma}\norm{\gamma}_{\ell^2}.
  \end{equation}
  Also 
  \begin{align}
    \sup_{h,h'\in\fH_Q'}\norm{h-h'}_{\infty} 
    & \leq \sup_{h\in\fH_Q'}2\sum_{\vk\in\sZ^{d}}\Abs{\fF[h-h_{\rm ini}](\vk)}\\
    & \leq 2\sup_{h\in\fH_Q'}\left[\Abs{\fF[h-h_{\rm ini}](\vzero)}+\left(\sum_{\vk\in\sZ^{d*}}(\gamma(\vk))^2\right)^{\frac{1}{2}}\left(\sum_{\vk\in\sZ^{d*}}(\gamma(\vk))^{-2}\Abs{\fF[h-h_{\rm ini}](\vk)}^{2}\right)^{\frac{1}{2}}\right]\\
    & \leq 2\norm{f'}_{\infty}+4\norm{f'}_{\gamma}\norm{\gamma}_{\ell^2}.
  \end{align}
  Then 
  \begin{equation}
    \RD(h_{n})\leq \frac{2}{\sqrt{n}}\norm{f'}_{\infty}+\frac{4}{\sqrt{n}}\norm{f'}_{\gamma}\norm{\gamma}_{\ell^2}+\left(4\norm{f'}_{\infty}+8\norm{f'}_{\gamma}\norm{\gamma}_{\ell^2}\right)\sqrt{\frac{2\log(4/\delta)}{n}}.
  \end{equation}
\end{proof}
\begin{rmk}
  By the assumption in the theorem, the target function $f$ belongs to $\fF_\gamma(\Omega)$ which is a subspace of $L^2(\Omega)$. In most applications, $f$ is also a continuous function. In any case, $f$ can be well-approximated by a large neural network due to universal approximation theory \citep{cybenko1989approximation}.
\end{rmk}

Our a priori generalization error bound in Theorem \ref{thm:priorierror}
is large if the target function possesses significant high frequency
components. Thus, it explains the failure of DNNs in generalization for learning the
parity function \citep{shalev2017failures}, whose power concentrates at high
frequencies. In the following, We use experiments to illustrate that, as predicted by our a priori generalization error bound, larger FP-norm of the target function indicates a larger generalization error.

\section{Numerical experiments}\label{sec:exps}
In this section, we conduct numerical experiments to validate the effectiveness of LFP model for two-layer ReLU and Tanh networks. In addition, we would show that, with sufficient samples, the test error still increases as the frequency of the target function increases. The procedure to numerically solve the LFP model can be found in Appendix \ref{sec:numsolveopt}\footnote{The code can be found at \url{https://github.com/xuzhiqin1990/LFP}}.
\subsection{The effectiveness of LFP model}
Without the last term in Eq. \eqref{eq..lfpoperatorthm.ReLU} arising from the evolution of $\vw$, we would show that the simplified LFP model in \ref{eq..lfpoperatorthm.simpleReLU} can still predict the learning results of two-layer wide NNs. 

For 1d input example of ReLU NN, when the term of $1/\vxi^4$ dominates, as shown in Fig. \ref{fig:1drelu}(a), the NN interpolates training data by a smooth function (denoted by $f_{NN}$, red solid), which nearly overlaps with the prediction of LFP model (denoted by $f_{LFP}$ and the cubic spline interpolation (grey dashed). On the contrary,  when the term of $1/\vxi^2$ dominates, as shown in Fig. \ref{fig:1drelu}(b), the NN interpolates training data by a function, which nearly overlaps with the prediction of LFP model and the linear spline interpolation. These result are consistent with above analysis. 

For 2d input example of ReLU NN, we consider the XOR problem, which cannot be solved by one-layer neural networks \citep{minsky2017perceptrons}. The training samples consist of four
points represented by black stars in Fig. \ref{fig:2drelu}(a). As
shown in Fig. \ref{fig:2drelu}(b), our LFP model predicts accurately outputs of the well-trained NN over the input domain $[-1,1]^{2}$.

For two-layer Tanh NN, the weight coefficient decays exponentially w.r.t. the frequency no matter which part dominates, thus, the NN always learns the training data by a smooth function, as shown in Fig. \ref{fig:1dtanh}.

\begin{center}
\begin{figure}
\begin{centering}
\subfloat[]{\includegraphics[scale=0.5]{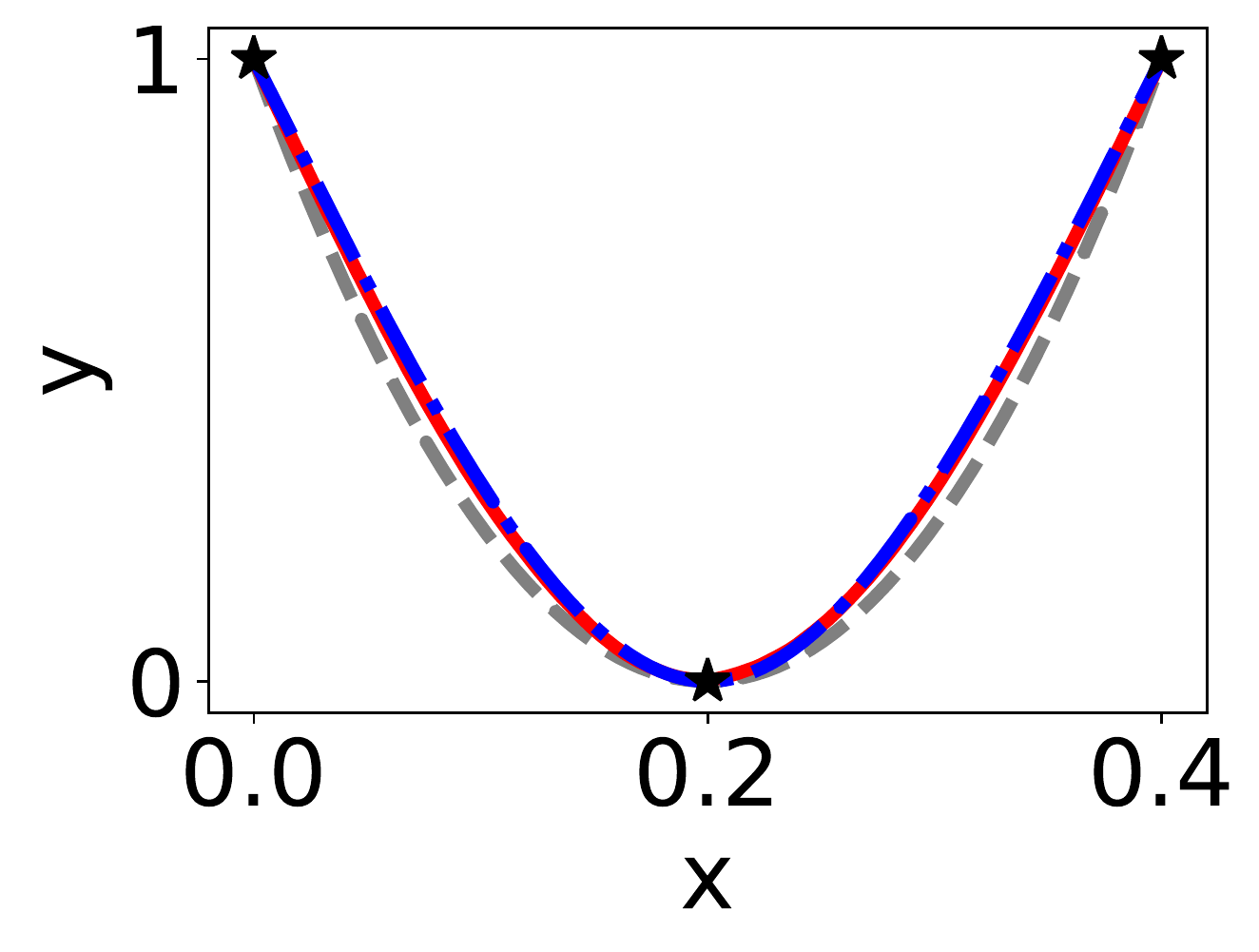}}
\subfloat[]{\includegraphics[scale=0.5]{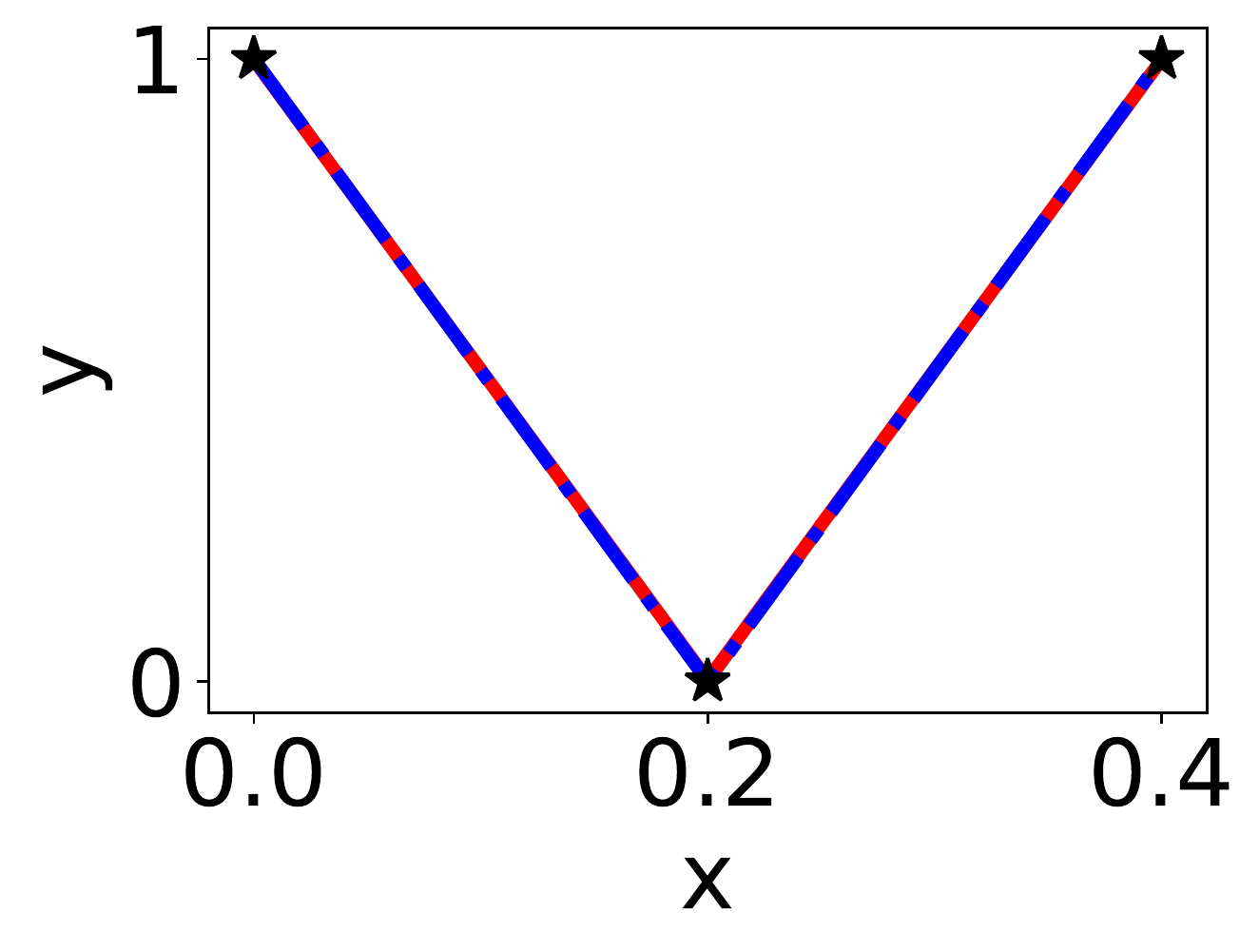}}
\par\end{centering}
\caption{$f_{\rm NN}$ (red solid) vs. $f_{\rm LFP}$ (blue dashed dot) vs. splines (grey dashed, cubic spline for (a) and linear spline for (b)) for a $1$-d problem. All curves nearly overlap with one other. Two-layer ReLU NN of $10000$ hidden neurons is initialized with (a) $\left\langle r^{2}\right\rangle _{r}\gg \left\langle a^{2}\right\rangle _{a}$, and (b) $\left\langle r^{2}\right\rangle _{r}\ll \left\langle a^{2}\right\rangle _{a}$. Black stars indicates training data. \label{fig:1drelu} }
\end{figure}
\par\end{center}

\begin{center}
\begin{figure}
\begin{centering}
\subfloat[]{\includegraphics[scale=0.5]{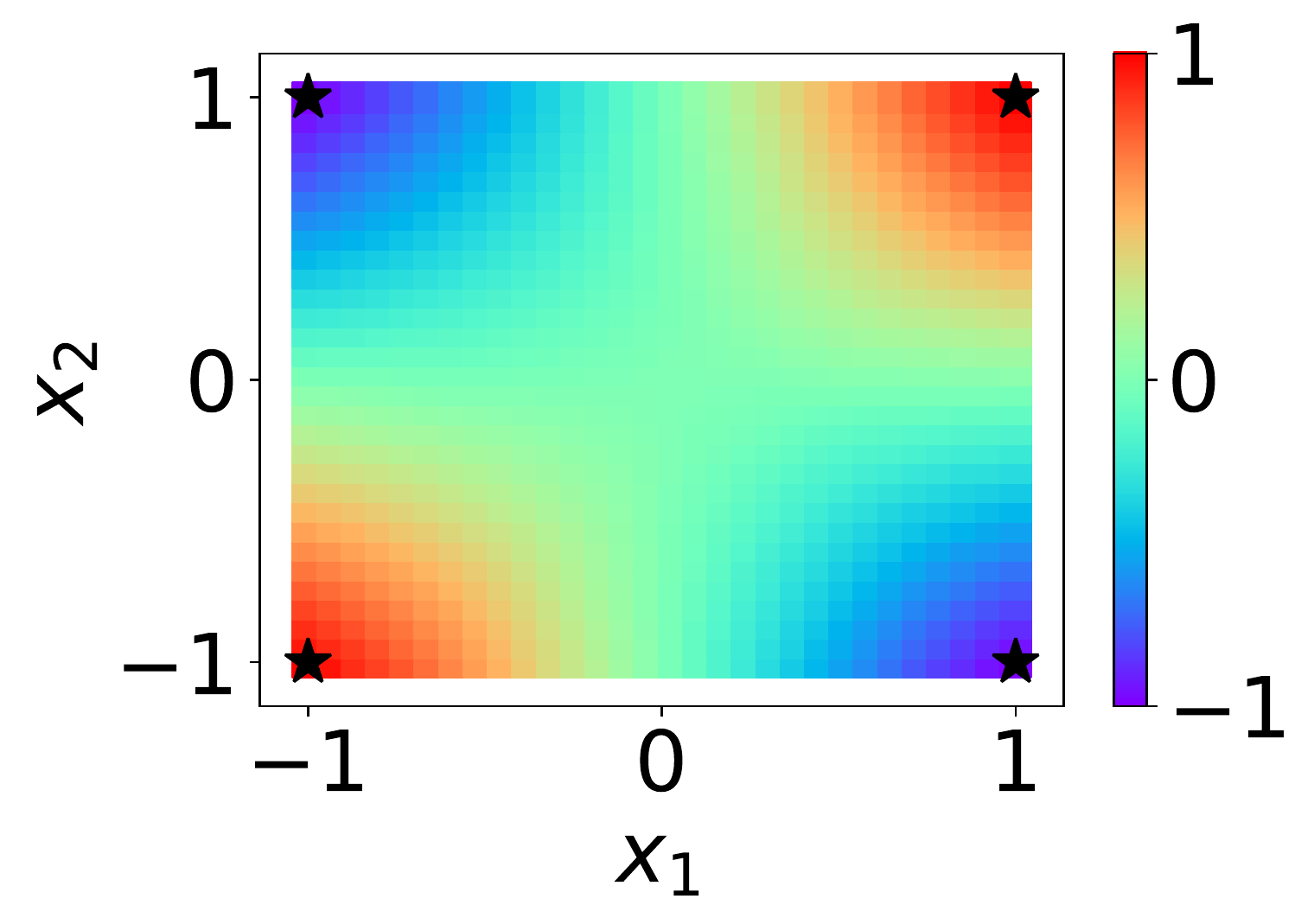}}
\subfloat[]{\includegraphics[scale=0.5]{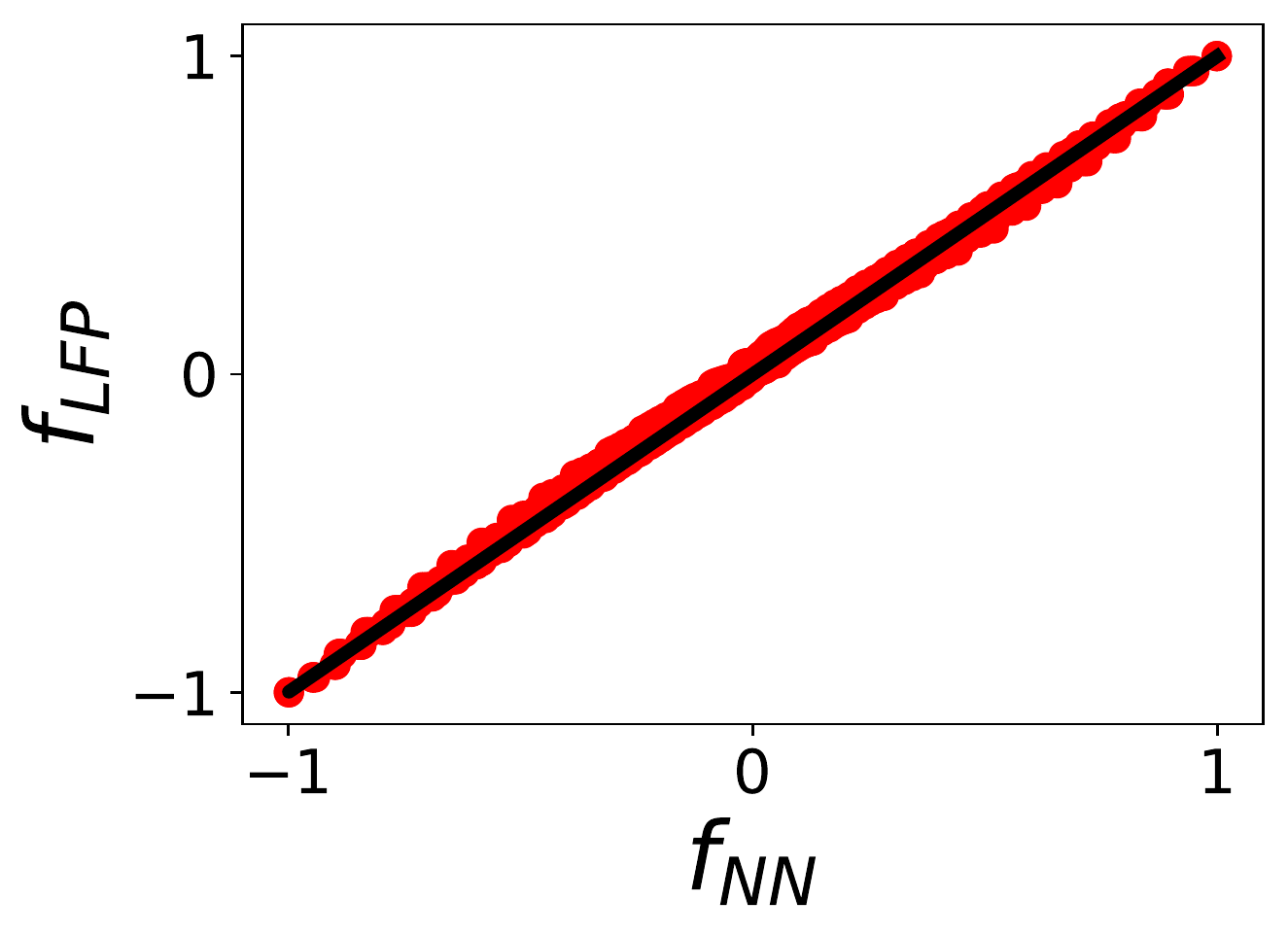}}
\par\end{centering}
\caption{$2$-d XOR problem with four training data indicated by black stars learned by a two-layer ReLU NN of $80000$ hidden neurons. (a) $f_{\rm NN}$ illustrated in color scale. (b) $f_{\rm LFP}$ (ordinate) vs. $f_{\rm NN}$ (abscissa) represented by red dots evaluated over whole input domain $[-1,1]^2$. The black line indicates the identity function.  \label{fig:2drelu} }
\end{figure}
\par\end{center}

\begin{center}
\begin{figure}
\begin{centering}
\subfloat[]{\includegraphics[scale=0.5]{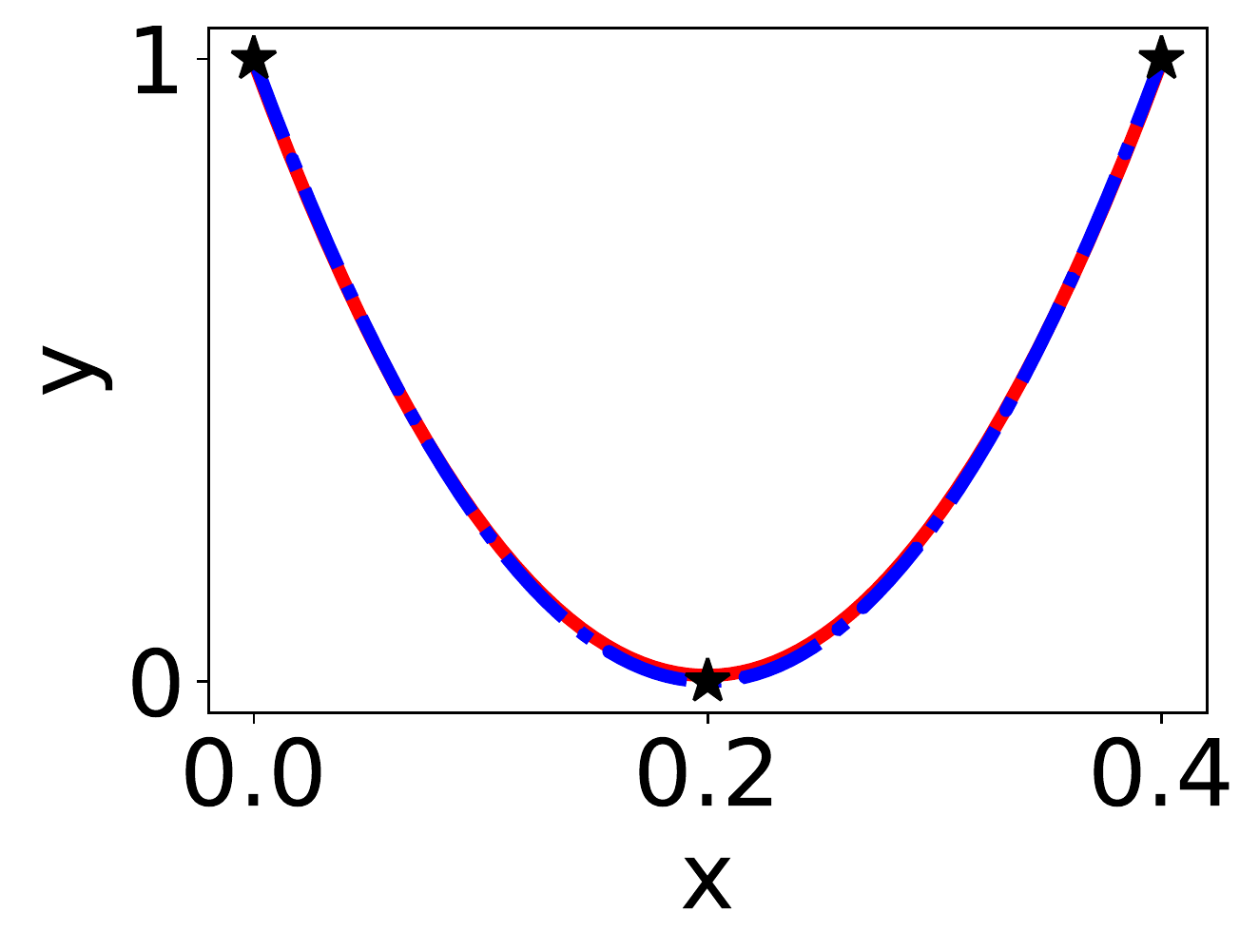}}
\subfloat[]{\includegraphics[scale=0.5]{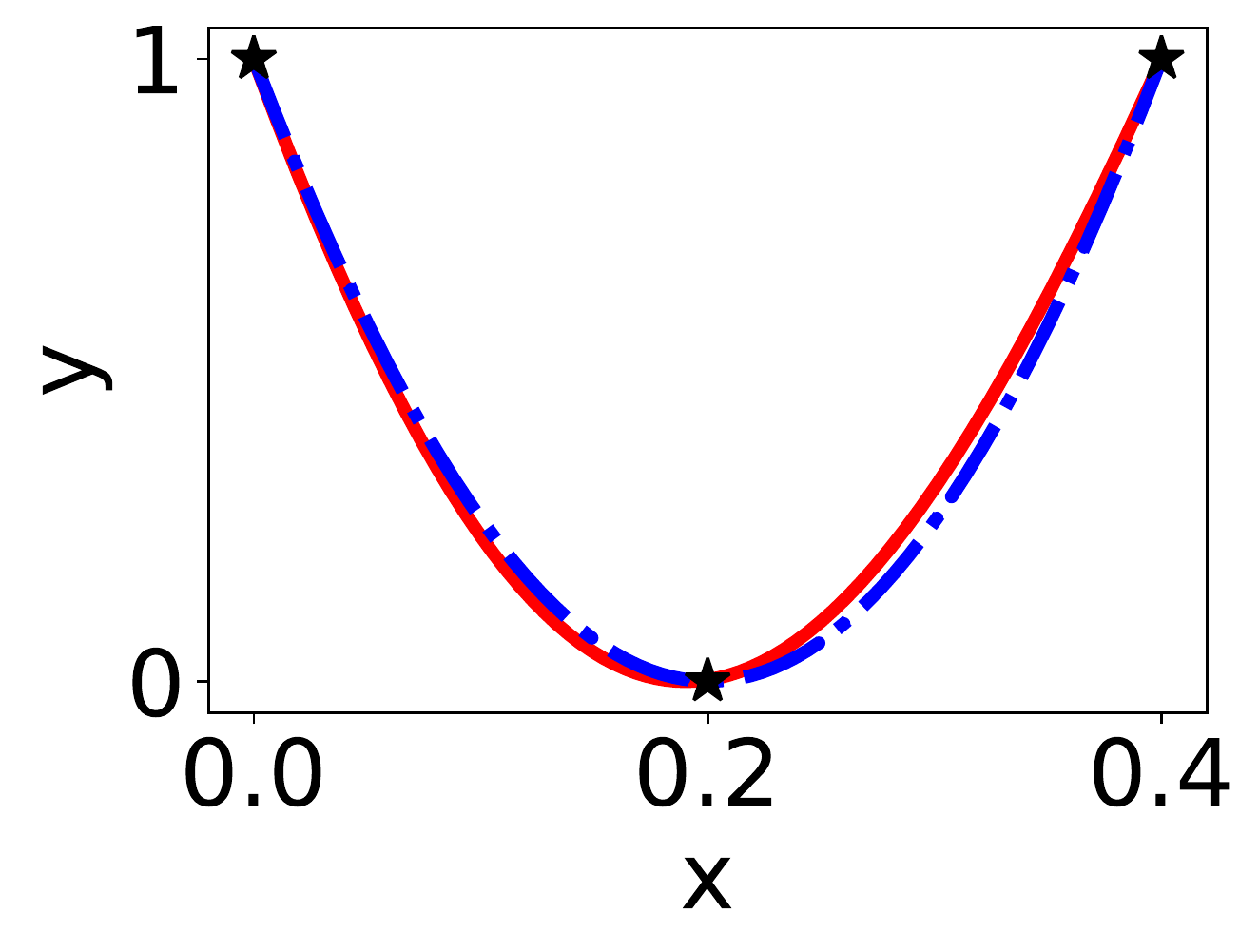}}
\par\end{centering}
\caption{$f_{\rm NN}$ (red solid) vs. $f_{\rm LFP}$ (blue dashed dot) for a $1$-d problem. Two-layer tanh NN of $10000$ hidden neurons is initialized with (a) $\left\langle r^{2}\right\rangle _{r}\gg \left\langle a^{2}\right\rangle _{a}$, and (b) $\left\langle r^{2}\right\rangle _{r}\ll \left\langle a^{2}\right\rangle _{a}$.  Black stars indicates training data. \label{fig:1dtanh} }
\end{figure}
\par\end{center}

\subsection{Generalization error}

In this section, we train a ReLU-NN of width 1-5000-1 to fit 20
uniform samples of $f(x)=\sin(2\pi vx)$ on $[0,1]$ until the training
MSE loss is smaller than $10^{-6}$, where $v$ is the frequency. The number of training sample is sufficient to recover the frequency of the target function by the Nyquist sampling theorem. We then use 500 uniform samples
to test the NN. As the frequency of the target function increases, the FP-norm would increase, thus, leading to a looser bound of the generalization error. As shown in Fig. \ref{fig:fpnorm}, the test error increases as the frequency of the target function increases. 

\begin{center}
\begin{figure}
\begin{centering}
\includegraphics[scale=0.6]{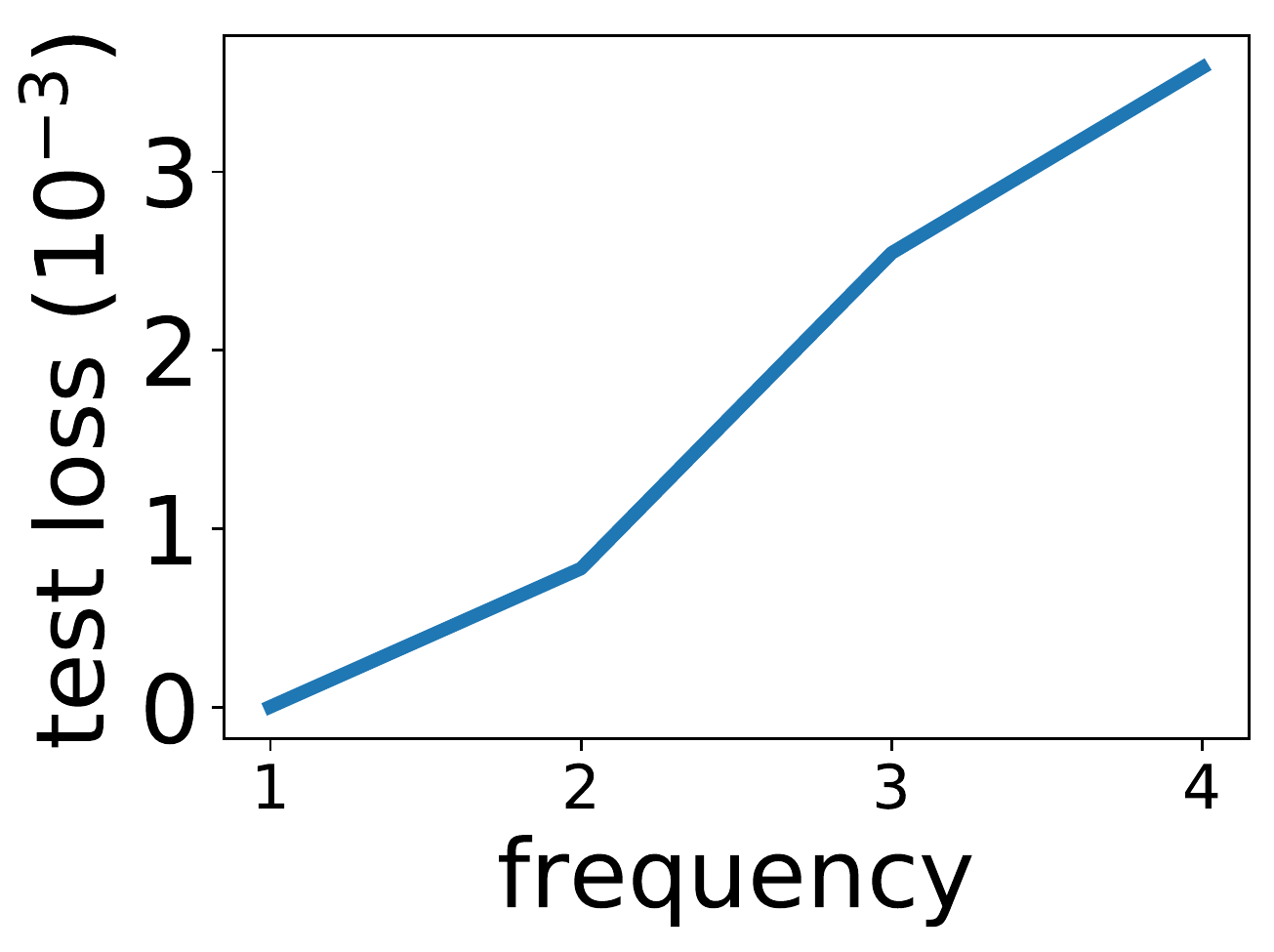} 
\par\end{centering}
\caption{Test loss are plotted as a function of frequency $v$ of the target function $\sin(2\pi vx)$.\label{fig:fpnorm} }
\end{figure}
\par\end{center}

\section{Discussion}\label{sec:discussion}

In this work, inspired by the F-Principle, we derive an
LFP model for two layer wide NNs --- a model quantitatively well predicts the output
of two-layer ReLU and tanh NNs in an extremely over-parameterized regime. We explicitize
the implicit bias of the F-Principle by a constrained optimization
problem equivalent to the LFP model. This explicitization leads to
an \textit{a priori} estimate of the generalization error bound, which
depends on the FP-norm of the target function. Note
that, our LFP model for other transfer functions can also be derived  similarly.

The LFP model advances
our qualitative/empirical understandings of the F-Principle to a quantitative
level. i) With ASI trick \citep{zhang_type_2019} offsetting the initial
DNN output to zero, the LFP model indicates that the F-Principle also
holds for DNNs initialized with large weights. Therefore, ``initialized
with small parameters'' \citep{xu_training_2018,xu2019frequency}
is not a necessary condition for the F-Principle. ii) Based on the training behavior of F-Principle, previous works \citep{xu_training_2018,xu2019frequency,rahaman2018spectral}
speculate that ``DNNs prefer to learn the training data by a low
frequency function''. With an equivalent optimization problem explicitizing
the F-Principle, this speculation is demonstrated theoretically by the LFP model.

Our \textit{a priori} generalization error bound increases as the FP-norm
of the target function increases. This explains several important
phenomena. First, DNNs fail to generalize well for the parity function
\citep{shalev2017failures}. \citet{xu2019frequency} shows that this
is due to the inconsistency between the high frequency dominant property
of the parity function and the low frequency preference of DNNs. In
this work, by our \textit{a priori} generalization error bound, the dominant
high frequency of the parity function quantitatively results in a
large FP-norm, thus, a large generalization error. Second, because
randomly labeled data possesses large high frequency components, which
induces a large FP-norm of any function well matches the training
data and test data, we expect a very large generalization error, e.g., no generalization,
as observed in experiments. Intuitively, our estimate indicates
good generalization of NNs for well-structured low-frequency
dominant real dataset as well as bad generalization of NNs for randomly labeled data, thus providing
insight into the well known puzzle of generalization of DNNs \citep{zhang2016understanding}.

The F-Principle, a widely observed implicit bias of DNNs, is also
a natural bias for human. Empirically, when a human see several
points of training data, without a specific prior, one tends to interpolate
these points by a low frequency dominant function. Therefore, the success of
DNN may partly result from its adoption of a similar interpolation
bias as human's. In general, there could be multiple types of implicit biases underlying the training dynamics of a DNN. Inspired
by the LFP model, discovering and explicitizing these implicit biases
could be a key step towards a thorough quantitative understanding
of deep learning.








\appendix
\section{Fourier transform table \label{FTtable}}
We list some useful and well-known results for one-dimensional Fourier
transform in Table \ref{1dtable}.
\begin{table}[h]
    \renewcommand{\arraystretch}{1.8}
    \centering
    \begin{tabular}{lc}
        \toprule
        Function of $x$    & Fourier transform with respect to $x$                    \\
        \midrule
        $g(ax)$            & $\frac{1}{\abs{a}}\fF[g](\frac{\xi}{a})$                 \\
        $g(x-c)$           & $\fF[g](\xi)\E^{-2\pi\I c\xi}$                           \\
        $x^k g(x)$         & ${(\frac{\I}{2\pi})}^k \frac{\D^k}{\D \xi^k}\fF[g](\xi)$ \\
        $g^{(k)}(x)$       & ${(2\pi\I\xi)}^k\fF[g](\xi)$                             \\
        $1$                & $\delta(\xi)$                                            \\
        $x^k$              & ${(\frac{\I}{2\pi})}^{k}\delta^{(k)}(\xi)$               \\
        $\delta(x-x_0)$    & $\E^{-2\pi\I x_0 \xi}$                                   \\
        $H(x)$ (Heaviside) & $\frac{1}{\I2\pi\xi}+\frac{1}{2}\delta(\xi)$             \\
        $\ReLU(x)$         & $-\frac{1}{4\pi^2\xi^{2}}+\frac{\I}{4\pi}\delta'(\xi)$   \\
        $\tanh(x)$         & $-\I\pi\csch(\pi^2\xi)$                                  \\
        $\Sigmoid(x)$      & $-\I\pi\csch(2\pi^2\xi)+\frac{1}{2}\delta(\xi)$          \\
        $\sech^{2}(x)$     & $2\pi^2\xi\csch(\pi^2\xi)$                               \\
        $x\sech^{2}(x)$    & $\I\pi\left(1-\pi^2\xi\coth(\pi^2\xi)\right)
            \csch(\pi^2\xi)$                                                          \\
        \bottomrule
    \end{tabular}
    \caption{Fourier transform for $1$-dimensional functions.\label{1dtable}}
\end{table}
We also list
some useful and well-known results for high-dimensional Fourier transform in Table \ref{2dtable}.
\begin{table}[h]
    \centering
    \renewcommand{\arraystretch}{1.8}
    \begin{tabular}{lc}
        \toprule
        Function of $\vx$                     & Fourier transform with respect to $\vx$               \\
        \midrule
        $g(a\vx)$                             & $\frac{1}{\abs{a}^d}\fF[g](\frac{\vxi}{a})$           \\
        $\delta(\vx-\vx_0)$                   & $\E^{-2\pi\I\vxi^\T\vx_0 }$                           \\
        $g(\vnu^\T\vx)$ (unit vector $\vnu$)  & $\delta_{\vnu}(\vxi)
            \fF[g](\vxi^\T\vnu)$                                                                      \\
        $g(\vw^\T\vx+b)$                      & $\delta_{\vw}(\vxi)\fF[g](
            \frac{\vxi^\T\hat{\vw}}{\norm{\vw}})\E^{2\pi\I\frac{b}{\norm{\vw}}
                \vxi^\T\hat{\vw}}$                                                                    \\
        $g(\vw^\T\vx+\norm{\vw}c)$            & $\delta_{\vw}(\vxi)\fF[g](
            \frac{\vxi^\T\hat{\vw}}{\norm{\vw}})\E^{2\pi\I c\vxi^\T\hat{\vw}}$                        \\
        $\vx g(\vx)$                          & $\frac{\I}{2\pi}\nabla \fF[g](\vxi)$                  \\
        $\vx^{\perp}g(\vw^\T\vx+\norm{\vw}c)$ & $\frac{\I}{2\pi}
            \nabla_{\vxi^{\perp}}\left[\delta_{\vw}(\vxi)\fF[g](
                \frac{\vxi^\T\hat{\vw}}{\norm{\vw}})\E^{2\pi\I c\vxi^\T\hat{\vw}}\right]$             \\
        $\vx g(\vw^\T\vx+b)$                  & $\frac{\I}{2\pi}\nabla_{\vxi}\left[\delta_{\vw}(\vxi)
                \fF[g](\frac{\vxi^\T\hat{\vw}}{\norm{\vw}})\E^{2\pi\I b\vxi^\T\hat{\vw}/
                    \norm{\vw}}\right]$                                                               \\
        \bottomrule
    \end{tabular}
    \caption{Fourier transform for $d$-dimensional functions\label{2dtable}}
\end{table}

\section{Numerically solve the optimization problem}\label{sec:numsolveopt}
Numerically, we solve the following problem 
\begin{equation}
\min_{a_{n},b_{n}}\sum_{i=1}^{M}\left(\sum_{j\in I}\left[a_{j}\sin(2\pi\frac{j}{L'}x_{i})+b_{j}\cos(2\pi\frac{j}{L'}x_{i})\right]-y_{i}\right)^{2}+\epsilon\sum_{j\in I}w(2\pi\frac{j}{L'})^{-1}\left(a_{j}^{2}+b_{j}^{2}\right),
\end{equation}
where we set $I=\{0,\cdots,\frac{L'}{L}K-1\}$, $L'=10L$, $L$ is the range of the training inputs,
$K=200$ which is much larger than the number of training samples,
$\epsilon=10^{-6}$. We can rewrite the above problem into the vector
form 
\begin{equation}
\min_{\boldsymbol{a}}\left(E\boldsymbol{a}-Y\right)^{T}\left(E\boldsymbol{a}-Y\right)+\epsilon\boldsymbol{a}^{T}W^{-1}\boldsymbol{a},
\end{equation}
where 
\[
\boldsymbol{a}=[a_{0},\cdots,a_{\frac{L'}{L}K-1},b_{0},\cdots,b_{\frac{L'}{L}K-1}]^{T},
\]
\[
E=[\sin(2\pi\frac{0}{L'}X),\cdots,\sin(2\pi(\frac{K}{L}-1)X),\cos(2\pi\frac{0}{L'}X),\cdots,\cos(2\pi(\frac{K}{L}-1)X)],
\]
\[
X=[x_{1},\cdots,x_{M}]^{T},\quad Y=[y_{1},\cdots,y_{M}]^{T}.
\]
The
solution of the above problem satisfies 
\begin{equation}
E^{T}\left(E\boldsymbol{a}-Y\right)+\epsilon W^{-1}\boldsymbol{a}=0.
\end{equation}
Then $\boldsymbol{a}$ is solved as 
\begin{equation}
\boldsymbol{a}=\left[E^{T}E+\epsilon W^{-1}\right]^{-1}E^{T}Y.
\end{equation}

\section*{Acknowledgements}
Z.X. is supported by National Key R\&D Program of China (2019YFA0709503), and Shanghai Sailing Program. This work is also partially supported by HPC of School of Mathematical Sciences at Shanghai Jiao Tong University.
\bibliography{DLRef}
\end{document}